\def\year{2022}\relax
\documentclass[letterpaper]{article} 
\usepackage{aaai22}  
\usepackage{Symbol}
\usepackage{times}  
\usepackage{helvet}  
\usepackage{courier}  
\usepackage[hyphens]{url}  
\usepackage{graphicx} 
\usepackage{color, colortbl}
\usepackage{natbib}  
\usepackage{caption} 
\DeclareCaptionStyle{ruled}{labelfont=normalfont,labelsep=colon,strut=off} 
\definecolor{Gray}{gray}{0.9}
\definecolor{LightCyan}{rgb}{0.88,1,1}
\usepackage{algorithm}
\usepackage{algorithmic}

\usepackage{newfloat}

\title{Achieving Zero Constraint Violation for Constrained Reinforcement Learning via Primal-Dual Approach}
\author{
     Qinbo Bai\textsuperscript{\rm 1},
     Amrit Singh Bedi\textsuperscript{\rm 2},
     Mridul Agarwal\textsuperscript{\rm 1},
     Alec Koppel\textsuperscript{\rm 2},
     Vaneet Aggarwal\textsuperscript{\rm 1}
}
\affiliations{
\textsuperscript{\rm 1} Purdue University\\
\textsuperscript{\rm 2} US Army Research Laboratory \\
bai113@purdue.edu, amrit0714@gmail.com, agarw180@purdue.edu, alec.e.koppel.civ@mail.mil, vaneet@purdue.edu
}

\usepackage{amsmath}
\usepackage{enumitem}
\usepackage{amsthm}
\usepackage{amsfonts}
\usepackage{threeparttable}
\usepackage{subfigure}
\usepackage{tikz}
\usepackage{pgfplots}
\usepackage{pgf}
\usepackage{epstopdf}
\usepgflibrary{shapes}
\usetikzlibrary{%
	arrows,%
	decorations.text,%
	positioning,%
	scopes,%
	shapes%
}

\newtheorem{assumption}{Assumption}

\newtheorem{lemma}{Lemma}
\newtheorem{theorem}{Theorem}
\newtheorem{corollary}{Corollary}

\begin{document}

\maketitle
\begin{abstract}
Reinforcement learning is widely used in applications where one needs to perform sequential decisions while interacting with the environment. The problem becomes more challenging when the decision requirement includes satisfying some safety constraints. The problem is mathematically formulated as constrained Markov decision process (CMDP). In the literature, various algorithms are available to solve CMDP problems in a model-free manner to achieve $\epsilon$-optimal cumulative reward with $\epsilon$ feasible policies. An $\epsilon$-feasible policy implies that it suffers from constraint violation. An important question here is whether we can achieve $\epsilon$-optimal cumulative reward with zero constraint violations or not. To achieve that, we advocate the use of randomized primal-dual approach to solve the CMDP problems and propose a conservative stochastic primal-dual algorithm (CSPDA) which is shown to exhibit $\tilde{\mathcal{O}}\left(1/\epsilon^2\right)$ sample complexity to achieve $\epsilon$-optimal cumulative reward with zero constraint violations. In the prior works, the best available sample complexity for the $\epsilon$-optimal policy  with zero constraint violation is $\tilde{\mathcal{O}}\left(1/\epsilon^5\right)$. Hence, the proposed algorithm provides a significant improvement as compared to the state of the art. \footnote{This is the arXiv version with the Appendix for the AAAI 2022 paper with the same title. This work has been further extended with concave utilities and constraints, and is available in v2 of the current arXiv.}

\end{abstract}

\section{Introduction}

Reinforcement learning (RL) is a machine learning framework which learns to perform a task by repeatedly interacting with the environment. This framework is widely utilized in a wide range of applications such as robotics, communications, computer vision,  autonomous driving, etc. \cite{arulkumaran2017deep,kiran2021deep}. The problem is mathematically formulated as a Markov Decision Process (MDP) which constitute of a state, action, and transition probabilities of going from one state to the other after taking a particular action. On taking an action, a reward is achieved and the overall objective is to maximize the sum of discounted rewards. However, in various realistic environments, the agent needs to decide action where certain constraints need to be satisfied (e.g., average power constraint in wireless sensor networks \cite{buratti2009overview}, queue stability constraints \cite{xiang2015joint}, safe exploration \cite{moldovan2012safe}, etc.). The standard MDP equipped with the cost function for the constraints is called constrained Markov Decision process (CMDP) framework \cite{altman1999constrained}. It is well known by \cite{altman1999constrained} that the resulting CMDP problem can be equivalently written as a linear program (LP) and hence efficient algorithms are available in the literature. 
But to solve the LP, one needs  access to the transition probabilities of the environment, which is not available in realistic environment, and thus efficient approaches to develop model-free algorithms  for CMDP are required. 

Various algorithms are proposed in the literature to solve the CMDP problem without apriori knowledge of the transition probability (See Table 1 for comparisons). 
The performance of these algorithms is measured by the number of samples (number of state-action-state transitions) required to achieve $\epsilon$-optimal (objective sub-optimality) $\epsilon$-feasible (constraint violations) policies. An $\epsilon$-feasible policy means that the constraints are not completely satisfied by the obtained policy. However, in many applications, such as in power systems \cite{vu2020safe} or autonomous vehicle control \cite{wen2020safe}, violations of constraint could be catastrophic in practice. Hence, achieving optimal objective guarantees without constraint violation is an important problem and is the focus of the paper. More precisely, we ask the question, ``{\it Is it possible to achieve the optimal sublinear convergence rate for the objective while achieving  zero constraint violations for  CMDP problem without apriori knowledge of the transition probability?}"


We answer the above question in affirmative in this work.  We remark that the sample complexity result in this work exhibit tight dependencies on the cardinality of state and action spaces (cf. Table 1 in Appendix \ref{sec:app_compare}). The key contributions can be summarized as follows:

\begin{itemize}[leftmargin=*]
    \item To best of our knowledge, this work is the first attempt to solve CMDPs to achieve optimal sample complexity with zero constraint violation.  There exist one exception in the literature which achieves the zero constraint violation but at the cost of $\tilde{\mathcal{O}}\left(1/\epsilon^5\right)$ sample complexity to achieve $\epsilon$ optimal policy \cite{wei2021provably}. In contrast, we are able to achieve zero constraint violation with $\tilde{\mathcal{O}}\left(1/\epsilon^2\right)$ sample complexity.

    \item We utilized the idea of conservative constraints  in the dual domain to derive the zero constraint violations. Conservative constrains were used recently for showing zero constraint violations in online constrained convex optimization in \cite{9439069}, while the problem of CMDP is much more challenging than online constrained optimization.  The dual constraint violations are then used to derive the primal domain results utilizing the novel analysis unique to this work (cf. Sec. \ref{sec:policy}). We remark that directly applying the conservative constraint idea in the primal domain does not result in the optimal dependence on the discount factor. 
    
     \item The proposed algorithm utilizes adaptive state-action pair sampling (cf. Eq. \eqref{eq_sampled_Lagrange}), due to which the stochastic gradient estimates exhibit unbounded second order moments. This makes the analysis challenging, and  standard saddle point algorithms cannot be used. This difficulty is handled by using KL divergence as the performance metric for the dual update similar to \cite{Junyu}.
    
    \item We have performed proof of concept experiments to support the theoretical findings.
\end{itemize}




\section{Related work}

\textbf{Unconstrained RL}. In the recent years, reinforcement learning  has been well studied for unconstrained tabular settings. Different algorithms are compared based upon the sample complexity of the algorithm which describes the number of samples $T$ required to achieve an $\epsilon$ optimal policy. For the infinite horizon discounted reward setting, \cite{Lattimore2012lowerbound} modified the famous model-based UCRL algorithm \cite{jaksch2010near} to achieve the PAC upper bound of $\tdO\bigg(\frac{|\ccalS||\ccalA|}{(1-\gamma)^3\epsilon^2}\bigg)$ on the sample complexity. \cite{yuantao2021Qlearning} improved the model-free vanilla Q-learning algorithm to achieve the sample Complexity $\tdO\bigg(\frac{|\ccalS||\ccalA|}{(1-\gamma)^4\epsilon^2}\bigg)$. For the episodic setting with episode length of $H$, \cite{azar2017minimax} proposed the model-based UCBVI algorithm and achieved a sample complexity of $\tdO\left(\frac{H^3|\mathcal{S}||\mathcal{A}|}{\epsilon^2}\right)$ which is equivalent to the lower bound provided in the paper. Along the similar lines, \cite{jin2018qlearning} proposed a model-free UCB Q-learning and achieved the sample complexity of $\tdO(\frac{H^5|\mathcal{S}||\mathcal{A}|}{\epsilon^2})$. Above all, there exists a number of near-optimal algorithms (either model-based or model-free) in the unconstrained tabular settings for RL.

\noindent \textbf{Model-based Constrained RL}. Once the estimated transition model is either given or estimated accurately enough, it makes intuitive sense to utilize a model-based algorithm to solve the constrained RL (CRL) problem because the problem boils down to solving only a linear program \cite{altman1999constrained}. Under the model-based framework, the authors \cite{efroni2020exploration} proposed $4$ algorithms namely OptCMDP \& OptCMDP-bonus, OptDual, and OptPrimalDual which solve the problem in the primal, dual, and primal-dual domains, respectively.  \cite{brantley2020constrained} proposed a modular algorithm, CONRL, which utilizes the principle of optimism and can be applied to standard CRL setting and also extended to the concave-convex and knapsack setting. {\cite{Kalagarla_Jain_Nuzzo_2021} proposed the UC-CFH algorithm which also works using the optimism principle and provided a PAC analysis for their algorithm. \cite{ding2021provably} considered a linear MDP with constraints setting and proposed the OPDOP algorithm and extended it to the tabular setting as well.}

\noindent \textbf{Model-free CRL}. As compared to the model-based algorithms,  existing results for the model-free algorithms are fewer. The  authors of \cite{achiam2017constrained} proposed a constrained policy optimization (CPO) algorithm and authors of \cite{tessler2018reward} proposed a reward constrained policy optimization (RCPO) algorithm. The authors of \cite{gattami2021reinforcement} related CMDP to  zero-sum Markov-Bandit games, and provided efficient solutions for CMDP. However, these works did not provide any convergence rates for their algorithms. Furthermore, the authors in \cite{ding2020natural} proposed a primal-dual natural policy gradient algorithm both in tabular and general settings and have provided a regret and constraint violation analysis. A primal only constraint rectified policy optimization (CRPO) algorithm is proposed in \cite{xu2021crpo} to achieve sublinear convergence rate to the global optimal policy and  sublinear convergence rate for the constraint violations as well. Most of the existing approaches with specific sample complexity and constraint violation error bound are summarized in Table \ref{table2}. Recently, \cite{chen2021primal} translated the constrained RL problem into a saddle point problem and proposed a primal-dual algorithm which achieved $\tdO(1/\epsilon^2)$ sample complexity to obtain $\epsilon$-optimal $\epsilon$- feasible solution.  However, the policy is considered as the primal variable in the algorithm and an estimation of Q-table is required in the primal update, which introduces extra sample complexity and computation complexity. 

\noindent \textbf{Online Constrained Convex Optimization}. In the field of standard online convex optimization with constraints, the problem of reducing the regret and constraint violation is well investigated in the recent years \cite{mahdavi2012trading,9439069}. Recently, the authors of  \cite{9439069} utilized the idea of conservative constraints to achieve $\epsilon$-optimal solution with $\tdO(1/\epsilon^2)$ sample complexity and zero constraint violations. We  utilize the conservative idea in this work to more complex setting of constrained RL problems to achieve zero constraint violations.


\section{Problem Formulation}
Consider an infinite horizon discounted reward constrained Markov Decision Process (CMDP) which is defined by tuple $(\ccalS,\ccalA,\bbP,\bbr, \bbg^i, I, \gamma,\bbrho)$. In this model, $\ccalS$ denotes the finite state space (with $|\mathcal{S}|$ number of states), $\ccalA$ is the finite action space (with $|\mathcal{A}|$ number of actions), and $\bbP:\ccalS\times\ccalA\rightarrow\Delta^{|\ccalS|}$  gives the transition dynamics of the CMDP (where $\Delta^d$ denotes the probability simplex in $d$ dimension). More specifically, $\bbP(\cdot\vert s,a)$ describes the probability distribution of next state conditioned on the current state $s$ and action $a$. We denote $\bbP(s'|s,a)$ as $\bbP_a(s,s')$ for simplicity. In the CMDP tuple,  $\bbr:\mathcal{S}\times\mathcal{A}\rightarrow [0,1]$ is the reward function, $\bbg^i:\mathcal{S}\times\mathcal{A}\rightarrow [-1,1]$ is the $i^{th}$ constraint cost function, and $I$ denotes the number of constraints. Further, $\gamma$ is the discounted factor and $\bbrho$ is the initial distribution of the states.

It is well known that there always exists a deterministic optimal policy for unconstrained MDP problem. However, the optimal policy for CMDP could be stochastic. Furthermore, \cite[Theorem 3.1]{altman1999constrained} shows that it is enough to consider stationary stochastic policies.  Thus, let us define the stationary stochastic policy as $\pi:\mathcal{S}\rightarrow\Delta^{\vert\mathcal{A}\vert}$, which maps a state to a distribution in the action space. The value functions for both reward and constraint's cost following such policy $\pi$ are given by \cite{chen2021primal} 
\begin{align}\label{value_function}
	V_\bbr^{\pi}(s)&=(1-\gamma)\bbE\bigg[\sum\nolimits_{t=0}^{\infty}\gamma^t \bbr(s_t,a_t)\bigg],\nonumber\\
	V_{\bbg^i}^{\pi}(s)&=(1-\gamma)\bbE\bigg[\sum\nolimits_{t=0}^{\infty}\gamma^t \bbg^i(s_t,a_t)\bigg],
\end{align}
for all $s\in\mathcal{S}$. At each instant $t$, for given state $s_t$ and action $a_t\sim\pi(\cdot|s_t)$, the next state $s_{t+1}$ is distributed as $s_{t+1}\sim\bbP(\cdot|s_t,a_t)$. The expectation in \eqref{value_function} is with respect to the transition dynamics of the environment and the stochastic policy $\pi$.  Let us denote $J_\bbr^{\pi}$ and $J_{\bbg^i}^{\pi}$ as the expected value function w.r.t. the  initial distribution such as
\begin{align}
    J_{\bbr,\bbrho}(\pi)=&\mbE_{s_0\sim\bbrho}[V_\bbr^{\pi}(s_0)]\;,
    \nonumber
    \\
    J_{\bbg^i,\bbrho}(\pi)=&\mbE_{s_0\sim\bbrho}[V_{\bbg^i}^{\pi}(s_0)], \ \ \forall i.
\end{align}
The goal here is to maximize the expected value function for reward $J_{\bbr,\bbrho}(\pi)$ with respect to policy $\pi$ subject to satisfying the constraints value function, formulated as 
\begin{equation}\label{eq:formulation_policy}
    \begin{aligned}
    \max_{\pi} &\quad J_{\bbr,\bbrho}(\pi)\\
    s.t. &\quad J_{\bbg^i,\bbrho}(\pi)\geq 0 \quad \forall i\in[I],
    \end{aligned}
\end{equation}
where $[I]$ denoted the index of constraints. We note that the problem in Eq. \eqref{eq:formulation_policy} optimizes in the policy space. However, the value function is a non-linear function with respect to policy. It is well known that the problem in \eqref{eq:formulation_policy} can be equivalently written in terms of a linear program in the occupancy measure space \cite{altman1999constrained}. Thus, we introduce the concept of occupancy measure as follows. For a given policy $\pi$, the occupancy measure is defined as
\begin{equation}\label{eq:definition_occupancy_measure}
    \bblambda(s,a)=(1-\gamma)\Big(\sum\nolimits_{t=0}^{\infty}\gamma^t\mathbb{P}(s_t=s, a_t=a)\Big),
\end{equation}
where $s_0\sim\bbrho$, $a_t\sim\pi(\cdot|s_T)$, $\mathbb{P}(s_t=s, a_t=a)$ is the probability of visiting state $s$ and taking action $a$ in step $t$. By the definition in \eqref{eq:definition_occupancy_measure}, the value functions in Eq. \eqref{eq:formulation_policy} can be expressed as
\begin{align}\label{linear}
	\mbE_{s\sim\bbrho}[V_\bbr^{\pi}(s)]=\left<\bblambda,\bbr\right>,\quad	\mbE_{s\sim\bbrho}[V_{\bbg^i}^{\pi}(s)]=\left<\bblambda,\bbg^i\right>.
\end{align}
Following the equivalence in \eqref{linear}, the original problem of \eqref{eq:formulation_policy} in policy space is equivalent to the following Linear program (LP) in the occupancy measure space given by \cite[Theorem 3.3(a)(b)]{altman1999constrained}
\begin{equation}\label{eq:formulation_occupany_measure}
    \begin{aligned}
    \max_{\bblambda\geq \bbzero} &\quad \left<\bblambda,\bbr\right>\\
    s.t. &\quad \left<\bblambda,\bbg^i\right>\geq0\quad \forall i\in[I],\\
    &\quad \sum\nolimits_{a\in\mathcal{A}}(\bbI-\gamma \bbP_a^T)\bblambda_a=(1-\gamma)\bbrho,
    \end{aligned}
\end{equation}
where $\bblambda_a=[\bblambda{(1,a)},\cdots,\bblambda{(|\ccalS|, a)}]\in\mathbb{R}^{|\mathcal{S}|}$ is the $a^{th}$ column of $\bblambda$. Notice that the equality constant in Eq. \eqref{eq:formulation_occupany_measure} sums up to 1, which means $\bblambda$ is a valid probability measure and we define $\Lambda:=\{\bblambda|\sum_{s,a}\bblambda{(s,a)}=1\}$ as a probability simplex. For a  given occupancy measure $\bblambda$, we can recover the policy $\pi_{\bblambda}$ as 
\begin{equation}
	\pi_{\bblambda}(a|s)=\frac{\bblambda{(s,a)}}{\sum_{a'}\bblambda{(s,a')}}.
\end{equation}
By \cite[Theorem 3.3(c)]{altman1999constrained}, it is known that if $\bblambda^*$ is the optimal solution for problem in Eq. \eqref{eq:formulation_occupany_measure},  then $\pi_{\bblambda^*}$ will be an optimal policy for problem in Eq. \eqref{eq:formulation_policy}.

\section{Algorithm Development}

The problem in \eqref{eq:formulation_policy} is well studied in the literature and various model-based algorithms are proposed \cite{ding2020natural,xu2021crpo}. All of the existing approaches are able to achieve an objective optimality gap of $\tilde{\mathcal{O}}(1/\epsilon^2)$ with constraint violations  of $\tilde{\mathcal{O}}(\epsilon)$ where $\epsilon$ is the accuracy parameter. Recently, the authors in \cite{wei2021provably} proposed a triple-Q algorithm to achieve zero constraint violations at the cost of achieving objective optimality gap of $\tilde{\mathcal{O}}(1/\epsilon^5)$. The goal here is to develop an algorithm to achieve zero constraint violation without suffering for the objective optimality gap. To do so, we consider the conservative stochastic optimization framework presented in \cite{9439069} and utilize it to propose a conservative version of the constrained MDPs problem in \eqref{eq:formulation_occupany_measure} as 
\begin{subequations}
\begin{alignat}{3}
\max_{\bblambda\geq \bbzero} &\quad \left<\bblambda,\bbr\right>\label{formulation_objective}
\\
\text{s.t.} &\quad \left<\bblambda,\bbg^i\right>\geq\kappa\quad \forall i\in[I],\label{formulation_conservative1}
\\
&\quad \sum_{a\in\mathcal{A}}(\bbI-\gamma \bbP_a^T)\bblambda_a=(1-\gamma)\bbrho,  \label{formulation_conservative3}
\end{alignat}
\label{eq:formulation_conservative}
\end{subequations}
\noindent where $\kappa>0$ is tuning parameter which controls the conservative nature for the constraints. The idea is to consider a tighter version (controlled by $\kappa$) of the original inequality constraint in \eqref{eq:formulation_occupany_measure} which allows us to achieve zero constraint violation for CMDPs which does not hold for any existing algorithm. We will specify the specific value of the parameter $\kappa$ later in the convergence analysis section (cf. Sec. \ref{convergence_analysis}).  Note that the conservative version of the problem in Eq. \eqref{eq:formulation_conservative} is still a LP and hence the strong duality holds, which motivates us to develop the primal-dual based algorithms to solve the problem in \eqref{eq:formulation_conservative}. By the KKT theorem, the problem in Eq. \eqref{eq:formulation_conservative} is equivalent to the following a saddle point problem which we obtain by writing the Lagrangian of \eqref{eq:formulation_conservative} as 
    \begin{align}
    \ccalL(\bblambda,\bbu,\bbv)=&\left<\bblambda,\bbr\right>+\sum\nolimits_{i\in[I]}u^i\left(\langle\bbg^i,\bblambda\rangle-\kappa\right)\nonumber
\\
    &+(1-\gamma)\left<\bbrho,\bbv\right>+\sum_{a\in\mathcal{A}}\bblambda_a^T(\gamma\bbP_a-\bbI)\bbv
    \\
    =&\left<\bblambda,\bbr\right>+\left<\bbu,\bbG^T\bblambda-\kappa\bbone\right>+(1-\gamma)\left<\bbrho,\bbv\right>\nonumber
\\
    &+\sum_{a\in\mathcal{A}} \left<\bblambda_a,(\gamma\bbP_a-\bbI)\bbv\right>,\label{eq_Lagrange}
    \end{align}
where $\bbu:=[u_1,u_2,\cdots,u^i]^T$ is a column vector of the dual variable corresponding to constraints in \eqref{formulation_conservative1}, $\bbv$ is the dual variable corresponding to equality constraint in \eqref{formulation_conservative3}, $\bbG:=[\bbg^1, \cdots, \bbg^{I}] \in\mathbb{R}^{\mathcal{S}\times\mathcal{A}\times I}$ collects all the $\bbg^i$'s corresponding to $I$ constraints in \eqref{formulation_conservative1}, and $\bbone$ is the all one column vector. From the Lagrangian in \eqref{eq_Lagrange}, the equivalent saddle point problem is given by 
\begin{equation}\label{eq:formulation_saddle}
    \max_{\bblambda\geq\bbzero}\min_{\bbu\geq 0, \bbv}\ccalL(\bblambda,\bbu,\bbv).
\end{equation}
Since the Lagrange function is linear w.r.t. both primal and dual variable, it is known that the saddle point can be solved by the primal-dual gradient descent \cite{nedic2009subgradient}. However, since we assume that the transition dynamics $P_a$ is unknown, then directly evaluating  gradients of  Lagrangian in \eqref{eq:formulation_saddle} with respect to primal and dual variables is not possible. To circumvent this issue, we resort to a randomized primal dual approach proposed in \cite{wang2020randomized} to solve the problem in a model-free stochastic manner. We assume the presence of a generative model which is a common assumption in control/RL applications. The generative model results the next state $s'$ for a given state $s$ and action $a$ in the model and provides a reward $\bbr(s,a)$ to train the policy. To this end, we consider a distribution $\bbzeta$ over $\ccalS\times\ccalA$ to write a stochastic approximation for the  Lagrangian $\ccalL(\bblambda,\bbu,\bbv)$ in \eqref{eq:formulation_saddle} as
	\begin{align}\label{eq_sampled_Lagrange}
	&\ccalL_{(s,a,s'),s_0}^{\bbzeta}(\bblambda,\bbu,\bbv)
	\\
	&\quad=
	(1-\gamma)\bbv{(s_0)}+ \bbone_{\bbzeta{(s,a)}>0}\cdot\frac{\bblambda{(s,a)}(Z_{sa}-M)}{\bbzeta(s,a)}-\sum\nolimits_{i\in[I]}\kappa u_i,\nonumber
	\end{align}
	where
	\begin{align}
Z_{sa}:=\bbr{(s,a)}+\gamma \bbv{(s')}-\bbv(s)+\sum_{i\in[I]}u_i \bbg^i{(s,a)}, 	\end{align}
and $s_0\sim\bbrho$, the current state action pair $(s,a)\sim\bbzeta$, and the next state $s'\sim\bbP(\cdot\vert s,a)$. We remark that $M$ in \eqref{gradient_lambda} is a shift parameter which is used in the convergence analysis. We notice that the stochastic approximation $\ccalL_{(s,a,s'),s_0}^{\bbzeta}(\bblambda,\bbu,\bbv)$ in \eqref{eq_sampled_Lagrange} is an unbiased estimator for the Lagrangian function in Eq. \eqref{eq_Lagrange} which implies that $\mbE_{\bbzeta\times\bbP(\cdot\vert s,a), \bbrho}[\ccalL_{(s,a,s'),s_0}^{\bbzeta}]=\ccalL(\bblambda,\bbu,\bbv)$ with $\text{supp}(\bbzeta)\subset \text{supp}(\bblambda)$. We could see $\bbzeta$ as a adaptive state-action pair distribution which helps to control the variance of the stochastic gradient estimator. The stochastic gradients of the Lagrangian with respect to primal and dual variables are given by
%
%
\begin{align}
    \hat{\nabla}_{\bblambda}\ccalL(\bblambda,\bbu,\bbv)&=\bbone_{\bbzeta{(s,a)}>0}\cdot\frac{Z_{sa}-M}{\bbzeta{(s,a)}}\cdot\bbE_{sa},\label{gradient_lambda}
    \\
    \hat{\nabla}_{\bbu}\ccalL(\bblambda,\bbu,\bbv)&=\bbone_{\bbzeta{(s,a)}>0}\cdot\frac{\bblambda{(s,a)}\bbg{(s,a)}}{\bbzeta{(s,a)}}-\kappa\bbone,\label{eq:estimated_gradient}
    \\
    \hat{\nabla}_{\bbv}\ccalL(\bblambda,\bbu,\bbv)&\!=\!\bbe({s_0}')\!+\!\bbone_{\bbzeta{(s,a)}\!>\!0}\cdot\frac{\bblambda{(s,a)}(\gamma\bbe{(s')}-\bbe{(s)})}{\bbzeta{(s,a)}},
\end{align}
 where  we define $\bbe({s_0}')=(1-\gamma)\bbe({s_0})$ with $\bbe{(s_0)}\in\mbR^{|\ccalS|}$ being a column vector with all entries equal to $0$ except only the $s^{th}$ entry equal to $1$, $\bbE_{sa}\in\mbR^{|\ccalS|\times|\ccalA|}$ is a matrix with only the $(s,a)$ entry equaling to 1 and all other entries being $0$, and $\bbg{(s,a)}=[\bbg^1{(s,a)},\cdots,\bbg^i{(s,a)}]^T$. 
 
  With all the stochastic gradient definitions in place, we are now ready to present the proposed novel algorithm called Conservative Stochastic Primal-Dual Algorithm (CSPDA) summarized in Algorithm \ref{alg:spdgd}.
\begin{algorithm}[ht]
	\caption{\textbf{C}onservative \textbf{S}tochastic \textbf{P}rimal-\textbf{D}ual \textbf{A}lgorithm (CSPDA) for constrained RL }
	\label{alg:spdgd}
	\textbf{Input}: Sample size T. Initial distribution $\bbrho$. Discounted factor $\gamma$.\\
	\textbf{Parameter}: Step-size $\alpha,\beta$. Slater variable $\varphi$, Shift-parameter $M$, Conservative variable $\kappa$ and Constant $\delta\in(0,\frac{1}{2})$\\
	\textbf{Output}: $\bbarlambda=\frac{1}{T}\sum_{t=1}^T\bblambda^t$, $\bbaru=\frac{1}{T}\sum_{t=1}^T\bbu^t$ and $\bbarv=\frac{1}{T}\sum_{t=1}^T\bbv^t$
	\begin{algorithmic}[1] 
		\STATE Initialize $\bbu^1\in\ccalU$, $\bbv^1\in\ccalV$ and $\bblambda^1=\frac{1}{|\ccalS||\ccalA|}\cdot\bbone$
		\FOR{$t=1,2,...,T$} 
		\STATE $\bbzeta^t:=(1-\delta)\bblambda^t+\frac{\delta}{|\ccalS||\ccalA|}\bbone$
		\STATE Sample $(s_t,a_t)\sim\bbzeta^t$ and $s_0\sim\bbrho$
		\STATE Sample $s_t'\sim\ccalP(\cdot|a_t,s_t)$ from the generative model and observe reward $r_{sa}$
		\STATE Update value functions as $\bbu$ and $\bbv$ as
		\begin{equation}\label{eq:update_u}
			\bbu^{t+1}=\Pi_{\ccalU}(\bbu^t-\alpha\hat{\nabla}_{\bbu}\ccalL(\bblambda^t,\bbu^t,\bbv^t))
		\end{equation}
		\begin{equation}\label{eq:update_v}
			\bbv^{t+1}=\Pi_{\ccalV}(\bbv^t-\alpha\hat{\nabla}_{\bbv}\ccalL(\bblambda^t,\bbu^t,\bbv^t))
		\end{equation}
		\STATE Update occupancy measure as
		\begin{align}\label{eq:update_lambda1}
			\bblambda^{t+\frac{1}{2}}=&\arg\max_{\bblambda}\left<\hat{\nabla}_{\bblambda}\ccalL(\bblambda^t,\bbu^t,\bbv^t),\bblambda-\bblambda^t\right>\nonumber
			\\
			&-\frac{1}{\beta}KL(\bblambda\Vert\bblambda^t)
			\\
			\bblambda^{t+1}=&{\bblambda^{t+\frac{1}{2}}}/{\Vert\bblambda^{t+\frac{1}{2}}\Vert_1}\label{eq:update_lambda2}
		\end{align}
		\ENDFOR
	\end{algorithmic}
\end{algorithm}
First, we initialize the primal and dual variables in step 1. In step 4 and 5, we sample $(s_t,a_t,s_0)$ and then obtain $s_t'$ from the generative model. In step 6, we update the dual variables by the gradient descent step and a projection opration (See Lemma \ref{lem:bounded_dual_u} for the definition of $\ccalU$ and $\ccalV$). In step 7, we utilize the mirror ascent update and utilize the KL divergence as the Bregman divergence to obtain tight dependencies on the convergence rate analysis similar to \cite{wang2020randomized}. Then, the occupancy measure is normalized so that it remains a valid distribution. 
\section{Convergence Analysis}\label{convergence_analysis}
In this section, we study the convergence rate of the proposed Algorithm \ref{alg:spdgd} in detail. We start by analyzing the duality gap for the saddle point problem in \eqref{eq:formulation_saddle}. Then we show that the output of Algorithm \ref{alg:spdgd} given by $\bar\bblambda$ is $\eps$-optimal for the conservative version of the dual domain optimization problem in \eqref{eq:formulation_conservative} of CMDPs. Finally, we perform the analysis in the policy space and present the main results of this work. We prove that the induced policy $\bbarpi$ by the optimal occupancy measure $\barblambda$ is also $\eps$-optimal and achieves zero constraint violation at the same time. 
%
Before discussing the convergence analysis, we provide a detailed description of the assumptions required for the work in this paper.   
\begin{assumption}(Strict feasibility)\label{ass:Slater}
	There exists a strictly feasible occupancy measure $\hat \bblambda \geq 0$ to problem in \eqref{eq:formulation_conservative} such that
	\begin{equation}
		\begin{aligned}
		&\left<\hat{\bblambda}, \bbg^i\right> -\varphi \geq 0 \quad\forall i\in[I]\\
		\text{and}\quad & \sum\nolimits_{a}(\bbI-\gamma \bbP_a^T)\hat{\bblambda}_a=(1-\gamma)\bbrho
		\end{aligned}
	\end{equation}
for some $\varphi>0$.
\end{assumption}

Assumption \ref{ass:Slater} is the stronger version of the popular Slater's condition which is often required in the analysis of convex optimization problems. A similar assumption is considered in the literature  as well \cite{mahdavi2012trading,9439069} and also helps to ensure the boundedness of dual variables (see Lemma \ref{lem:bounded_dual_u}). We remark that Assumption \ref{ass:Slater} is unique to utilize the idea of conservative constraints to obtain zero constraint violations in the long term. To be specific, Assumption \ref{ass:Slater} plays a crucial rule in proving that the optimal objective values of the original problem in (6) and it's conservative version in (8) are $\mathcal{O}(\kappa)$ apart as mentioned in Lemma \ref{lem:diff_conservative}.

\subsection{Convergence Analysis for Duality Gap}\label{sec:duality_gap}
In order to bound the duality gap, we note that the standard analysis of saddle point algorithms \cite{nedic2009subgradient,9439069} is not applicable because of the unbounded noise introduced into the updates due to the use of adaptive sampling of the state-action pairs \cite{wang2020randomized,Junyu}. Therefore, it becomes necessary to obtain explicit bounds on the gradient as well as the variance of the stochastic estimates of the gradients. We start the analysis by consider the form of Slater's condition in Assumption \ref{ass:Slater}, and show that the dual variables $\bbu$ and $\bbv$ are bounded (Note that the optimal dual variables now will be function of conservative variable $\kappa$ as well).
\begin{lemma}[Bounded dual variable $\bbu$ and $\bbv$]\label{lem:bounded_dual_u}
    Under the Assumption \ref{ass:Slater}, the optimal dual variables $\bbu^*_{\kappa}$ and $\bbv_{\kappa}^*$ are bounded. Formally, it holds that $\Vert\bbu^*_{\kappa}\Vert_1\leq \frac{2}{\varphi}$ and  $\Vert\bbv_{\kappa}^*\Vert_{\infty}\leq \frac{1}{1-\gamma}+\frac{2}{(1-\gamma)\varphi}$.  
\end{lemma}
The proof of Lemma \ref{lem:bounded_dual_u} is provided in Appendix \ref{proof_lemma_1}.  As a result, we define $\ccalU:=\big\{\bbu~|~\Vert\bbu\Vert_1\leq \frac{4}{\varphi}\big\}$ and $\ccalV:=\big\{\bbv~|~\Vert\bbv\Vert_\infty\leq 2[\frac{1}{1-\gamma}+\frac{2}{(1-\gamma)\varphi}]\big\}$. 
%
 %
 Since we have mathematically defined the set $\mathcal{U}$ and $\mathcal{V}$,  now we rewrite the saddle point formulation in \eqref{eq:formulation_saddle} as 
 \begin{equation}\label{eq:formulation_saddle2}
 	\max_{\bblambda\in\Lambda}\min_{(\bbu\in\mathcal{U}, \bbv\in\mathcal{V})}\ccalL(\bblambda,\bbu,\bbv).
 \end{equation}
 In the analysis presented next, we will work with the problem in \eqref{eq:formulation_saddle2}. First, we decompose the duality gap in Lemma \ref{lem:dual_gap_bound} as follows. 
\begin{lemma}[Duality gap]\label{lem:dual_gap_bound}
	For any dual variables $\bbu,\bbv$, let us define $\bbw=[\bbu^T,\bbv^T]^T$,  and consider $\bar\bbu, \bar\bbv, \bar\bblambda$ as defined in Algorithm \ref{alg:spdgd}, the duality gap can be bounded as
		\begin{align}\label{eq:duality_gap_decompostion}
			\ccalL&(\barbu,\barbv,\bblambda_{\bbkappa}^*)-\ccalL(\bbu,\bbv,\barblambda)\nonumber
			\\
			&\leq\frac{1}{T}\sum_{t=1}^{T}\bigg[\underbrace{\left<\nabla_{\bblambda}\ccalL(\bbw^t,\bblambda^t),\bblambda_{\bbkappa}^*-\bblambda^t\right>}_{(I)}
			\nonumber
			\\ &\hspace{2cm}	+\underbrace{\left<\nabla_{\bbw}\ccalL(\bbw^t,\bblambda^t),\bbw^t-\bbw\right>}_{(II)}\bigg].
		\end{align}
\end{lemma}
The bound on terms $I$ and $II$ in the statement of Lemma \ref{lem:dual_gap_bound} are provided in Lemma \ref{lem:bound_primal} and \ref{lem:bound_dual} in the Appendix \ref{proof_theorem_1} (see proofs in Appendix \ref{lemma_4} and \ref{lemma_5}, respectively). This helps to  prove the main result in Theorem \ref{thm:duality_gap}, which establishes the final bound on the duality gap as follows.
\begin{theorem}\label{thm:duality_gap}
    Define $(\bbu^\dagger,\bbv^\dagger):=\arg\min_{\bbu,\bbv}\ccalL(\bbu,\bbv,\barblambda)$. Recall $\bblambda^*_{\bbkappa}$ is the best solution for the conservative Lagrange problem. The duality gap of the Algorithm \ref{alg:spdgd} is bounded as
    	\begin{align}\label{eq:duality_gap_bound}
        	&\mbE[\ccalL(\barbu,\barbv,\bblambda_{\bbkappa}^*)-\ccalL(\bbu^\dagger,\bbv^\dagger,\barblambda)]\nonumber\\
    		&\hspace{1cm}\leq \ccalO\bigg(\sqrt{\frac{I|\ccalS||\ccalA|\log(|\ccalS||\ccalA|)}{T}}\cdot \frac{1}{(1-\gamma)\varphi}\bigg).
    	\end{align}
\end{theorem}
 The proof of Theorem \ref{thm:duality_gap} is provided in Appendix \ref{proof_theorem_1}. The result in Theorem \ref{thm:duality_gap} describes a sublinear dependence of the duality gap onto the state-action space cardinality upto a logarithmic factor. In the next subsection we utilize the duality gap upper bound to derive a bound on the objective suboptimality and the constraint violation separately.    
 
\subsection{Dual Objective and Constraint Violation}\label{sec:occupancy_measure}
Recall that the saddle point problem in Eq. \eqref{eq:formulation_saddle2} is an equivalent problem to Eq. \eqref{eq:formulation_occupany_measure} where the main difference arises due to the newly introduced conservativeness parameter $\kappa$. Thus, a convergence analysis for duality gap should imply the convergence in occupancy measure in Eq. \eqref{eq:formulation_conservative}. But before that, we need to characterize the gap between the original problem \eqref{eq:formulation_occupany_measure} and its conservative version in \eqref{eq:formulation_conservative}. The following Lemma \ref{lem:diff_conservative} shows that the gap is of the order of parameter $\kappa$.

\begin{lemma}\label{lem:diff_conservative}
	Under Assumption \ref{ass:Slater}, and condition $\kappa\leq \min\{\frac{\varphi}{2}, 1\}$, it holds that the difference of optimal values between original problem and conservative problem is $\ccalO(\kappa)$. Mathematically, it holds that
	$\left<\bblambda^*,\bbr\right>-\left<\bblambda^*_{\kappa},\bbr\right>\leq \frac{\kappa}{\varphi}$.
\end{lemma}
The proof of Lemma \ref{lem:diff_conservative} is provided in Appendix \ref{proof_lemdiff_conservative}. Using the statement of Lemma \ref{lem:diff_conservative} and Theorem \ref{thm:duality_gap}, we obtain the convergence result in terms of output occupancy measure in following Theorem \ref{thm:occupancy_measure}.
\begin{theorem}\label{thm:occupancy_measure}
    For any $0<\eps<1$, there exists a constant $\tdc_1$ such that if 
    \begin{equation}
        T\geq\max\bigg\{ 16,4\varphi^2,\frac{1}{\epsilon^2}\bigg\}\cdot\tdc_1^2\frac{I|\ccalS||\ccalA|\log(|\ccalS||\ccalA|)}{(1-\gamma)^2\varphi^2}
    \end{equation}
     set $\kappa=\frac{2\tdc_1}{1-\gamma}\sqrt{\frac{I|\ccalS||\ccalA|\log(|\ccalS||\ccalA|)}{T}}$ and $M=4[\frac{1}{\varphi}+\frac{1}{1-\gamma}+\frac{2}{(1-\gamma)\varphi}]$, then the constraints of the original problem in \eqref{eq:formulation_occupany_measure} satisfy: 
    \begin{subequations}\label{eq:bound_occupancy_measure11}
    	\begin{align}
    	&\mbE\left<\barblambda,\bbg^i\right>\geq \eps\varphi\quad\forall i\in[I]\label{eq:bound_occupancy_measure2},
    	\\
    	&\mbE\Big\Vert\sum_{a}(\gamma\bbP_a^T-\bbI)\barblambda_a+(1-\gamma)\bbrho\Big\Vert_{1}\leq(1-\gamma)\epsilon\varphi\label{eq:bound_occupancy_measure3}.
    	\end{align}
	\end{subequations}
Additionally, the objective sub-optimality of \eqref{eq:formulation_occupany_measure} is given by 
\begin{align}
		\mbE[\left<\bblambda^*,\bbr\right>-\left<\barblambda,\bbr\right>]&\leq 3\epsilon\label{eq:bound_occupancy_measure1}.
\end{align}
\end{theorem}
The proof of Theorem \ref{thm:occupancy_measure} is provided in Appendix \ref{proof_thm:occupancy_measure}. Next, we present the special case of Theorem \ref{thm:occupancy_measure} in the form of Corollary \ref{coro:non_zero_violation_occupancy_measure} (see proof in Appendix \ref{proof_cor_1}), which shows the equivalent results  for the case without conservation parameter, $\kappa=0$. 
\begin{corollary}[Non Zero-Violation  Case]\label{coro:non_zero_violation_occupancy_measure}
    Set $\kappa=0$. For any $\eps>0$, there exists a constant $\tdc_1$ such that if 
        $ T\geq \tdc_1^2\cdot\frac{I|\ccalS||\ccalA|\log(|\ccalS||\ccalA|)}{(1-\gamma)^2\varphi^2\eps^2}$
    then $\barblambda$ satisfies the constraint violation as 
    \begin{subequations}\label{eq:bound_occupancy_measure}
    	\begin{align}
    &	\mbE\left<\barblambda,\bbg^i\right>\geq -\eps\quad\forall i\in[I]\label{eq:bound_occupancy_measure5}\\
    &	\mbE\Big\Vert\sum_{a}(\gamma\bbP_a^T-\bbI)\barblambda_a+(1-\gamma)\bbrho\Big\Vert_{1}\leq(1-\gamma)\epsilon\varphi\label{eq:bound_occupancy_measure6},
    	\end{align}
	\end{subequations}
and the sub-optimality is given by
%
    	$\mbE[\left<\bblambda^*,\bbr\right>-\left<\barblambda,\bbr\right>]\leq\eps$. \label{eq:bound_occupancy_measure4}
\end{corollary}

The positive lower bound of $\epsilon\varphi$ in \eqref{eq:bound_occupancy_measure2} hints that $\barblambda$ is feasible (hence zero constraint violation). On the other hand,  the lower bound in \eqref{eq:bound_occupancy_measure5} is negative $-\epsilon$ which states that the constraints in the dual space may not be satisfied for $\barblambda$. Next, we show that how the result in Theorem \ref{thm:occupancy_measure} helps to achieve the zero constraint violation in the policy space.
%
%
\subsection{Convergence Analysis in Policy Space}\label{sec:policy}
We have established the convergence in the occupancy measure space in Sec. \ref{sec:occupancy_measure} and shown that $\bar\bblambda$ achieves an $\epsilon$-optimal $\epsilon$-feasible solution but the claim of zero constraint violation is still not clear. But a small violation in Eq. \eqref{eq:bound_occupancy_measure3} makes $\bar{\bblambda}$ to loose its physical meaning as discussed in \cite[Proposition 1]{Junyu}. Thus, to make the idea clearer and explicitly show the benefit of the conservative idea utilized in this work, we further present the results in the policy space. The bound in Eq. \eqref{eq:bound_occupancy_measure3} provides an intuition that the output occupancy measure is close to the optimal one and therefore, the induced policy should also be close to the optimal policy. Such a result is mathematically presented next in Theorem \ref{thm:policy}.


\begin{theorem}[Zero-Violation] \label{thm:policy}
    Under the condition in Theorem \ref{thm:occupancy_measure} the induced policy $\bbarpi$ by the output occupancy measure $\barblambda$ is an $\eps$-optimal policy and achieves 0 constraint violation. Mathematically, this implies that 
	\begin{subequations}
		\begin{align}
		J_{\bbr,\bbrho}(\pi^*)-\mbE[J_{\bbr,\bbrho}(\bbarpi)]&\leq \eps\\
		\mbE[J_{\bbg^i,\bbrho}(\bbarpi)]&\geq 0\quad \forall i\in[I].\label{first}
		\end{align}
	\end{subequations}
\end{theorem}
The proof of Theorem \ref{thm:policy} is provided in Appendix \ref{proof_thm:policy}. To get better idea about the importance of result in Theorem \ref{thm:policy}, we next present a Corollary \ref{coro:non_zero_violation_policy} (see proof in \ref{proof_cor_2}) which is a special case of Theorem \ref{thm:policy} for $\kappa=0$. 
\begin{corollary}[Non Zero-Violation  Case]\label{coro:non_zero_violation_policy}
    Under the condition in Corollary \ref{coro:non_zero_violation_occupancy_measure}, the induced policy $\bbarpi$ by the output occupancy measure $\barblambda$ is an $\eps$-optimal policy w.r.t both objective and constraints. More formally,
	\begin{subequations}
		\begin{align}
		J_{\bbr,\bbrho}(\pi^*)-\mbE[J_{\bbr,\bbrho}(\bbarpi)]&\leq \eps\\
		\mbE[J_{\bbg^i,\rho}(\bbarpi)]&\geq -\eps\quad \forall i\in[I]. \label{second}
		\end{align}
	\end{subequations}
\end{corollary}
The benefit of utilizing the conservation parameter $\kappa$ becomes clear after comparing the results in \eqref{first} and \eqref{second}.

\section{ Evaluations on a Queuing System} \label{sec:simulation}
In this section, we evaluate the proposed Algorithm \ref{alg:spdgd} on a queuing system with a single server in discrete time \cite{altman1999constrained}[Chapter 5]. In this model, we assume a buffer of finite size $L$. A possible arrival is assumed to occur at the beginning of the time slot. The state of the system is the number of customers waiting in the queue at the beginning of time slot such that the size of state space is $\vert S\vert=L+1$. We assume that there are two kinds of actions:  \text{service action} and \text{flow action}. The service action is selected from a finite finite subset $\mathcal{A}$ of $[a_{min},a_{max}]$ such that $0<a_{min}\leq a_{max}<1$. With a service action $a$, we assume that a service of a customer is successfully completed with probability $a$. If the service succeeds, the length of the queue will reduce by one, otherwise queue length remains the same. The flow  action is a finite subset $\mathcal{B}$ of $[b_{min}, b_{max}]$ such that $0\leq b_{min}\leq b_{max}<1$. Given a flow action $b$, a customer arrives with probability $b$. Let the state at time $t$ be $x_t$, and we assume that no customer arrives when state $x_t=L$. Finally, the overall action space is the product of service action space and flow action space, i.e., $\mathcal{A}\times \mathcal{B}$. Given an action pair $(a,b)$ and current state $x_t$, the transition of this system $P(x_{t+1}|x_t,a_t=a,b_t=b)$ is shown in Table \ref{transition}.

\begin{table*}[h]   
	\begin{center}  				  
		{  
			\begin{tabular}{|c|c|c|c|}  
				\hline  
				Current State & $P(x_{t+1}=x_t-1)$ & $P(x_{t+1}=x_t)$ & $P(x_{t+1}=x_t+1)$ \\ \hline
				$1\leq x_t\leq L-1$ & $a(1-b)$ & $ab+(1-a)(1-b)$ & $(1-a)b$ \\ \hline
				$x_t=L$ & $a$ & $1-a$ & $0$ \\ \hline
				$x_t=0$ & $0$ & $1-b(1-a)$ & $b(1-a)$ \\ 
				\hline  
			\end{tabular}  
		}
	\end{center}  
	\caption{Transition probability of the queue system}  
	\label{transition}
\end{table*}

\begin{figure}[h]
    \centering
    \hspace{-0.4mm}
    \includegraphics[scale=0.48]{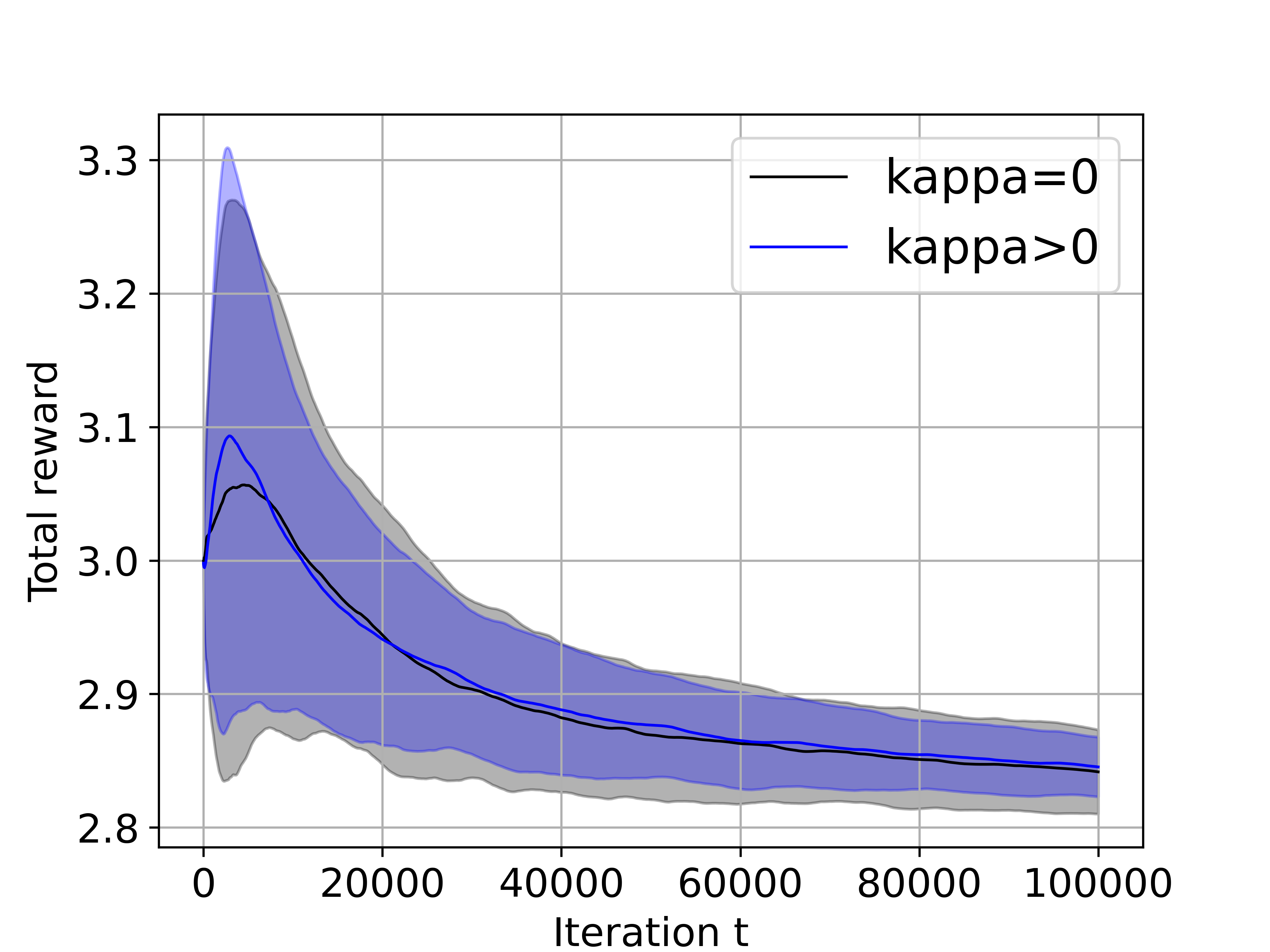}
    \caption{\small Learning Process of the proposed algorithm for objective with $\kappa=0$ and $\kappa>0$. The total reward is the objective in \eqref{simulation_prob}.}
    \label{fig:compare1}
\end{figure}
\begin{figure}[h]
	\centering
	\hspace{-0.4mm}
	\includegraphics[scale=0.48]{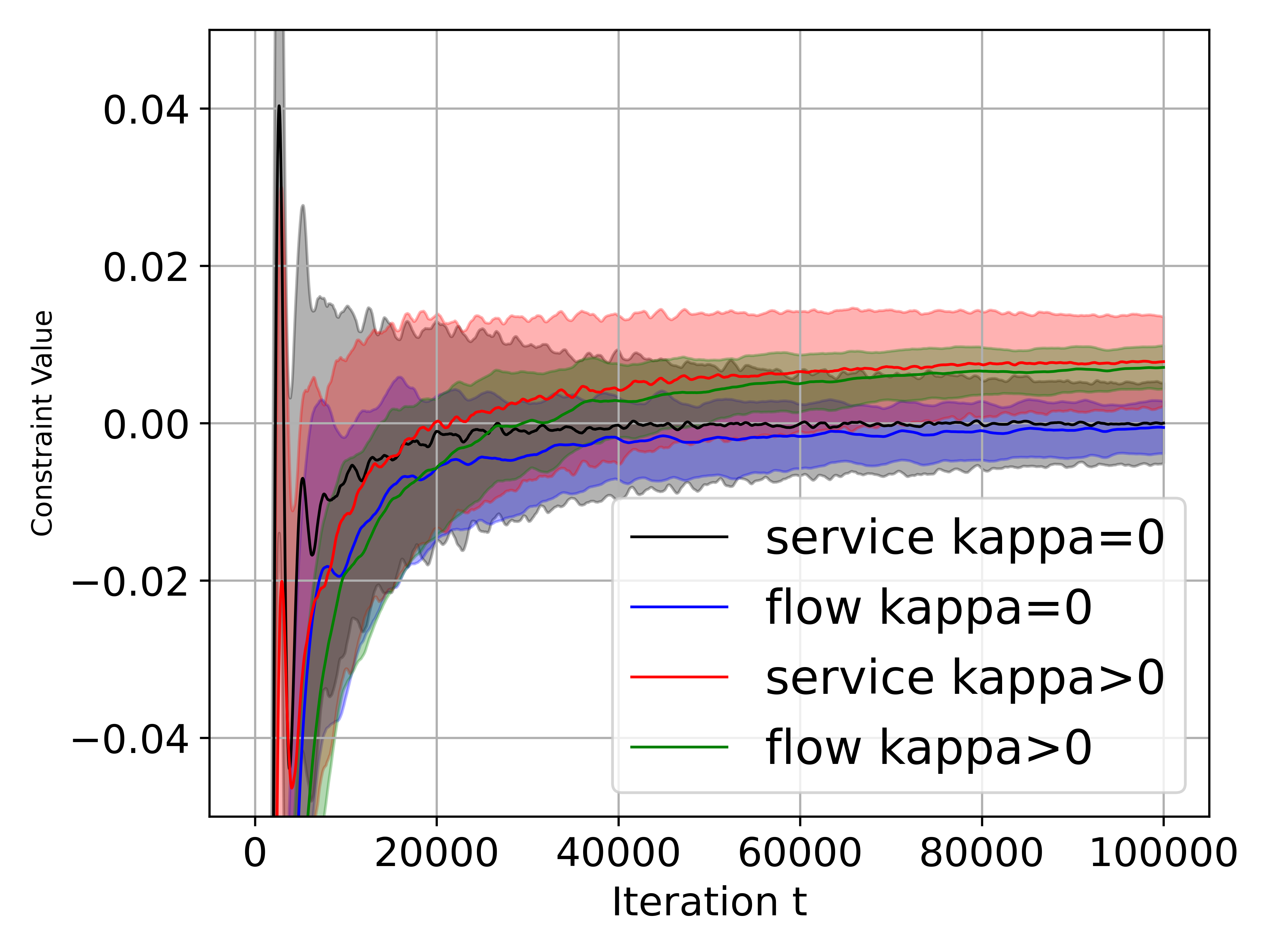}
	\caption{\small Learning Process of the proposed algorithm for constraint value with $\kappa=0$ and $\kappa>0$. The constraint value is the L.H.S. of the constraint in \eqref{simulation_prob}. The constraint value plot with error bars is plotted in Appendix \ref{appendix_sim}.}
	\label{fig:compare2}
\end{figure}
Assuming $\gamma=0.5$, we want to maximize the total discounted cumulative reward while satisfying two constraints with respect to service and flow, simultaneously. Thus, the overall optimization problem is given as 
\begin{align}\label{simulation_prob}
	\min\limits_{\pi^a, \pi^b}&\quad\mathbb{E}\bigg[\sum_{t=0}^{\infty}\gamma^tc(s_t, \pi^a(s_t), \pi^b(s_t))\bigg]\\
	s.t.&\quad
	\mathbb{E}\bigg[\sum_{t=0}^{\infty}\gamma^tc^i(s_t, \pi^a(s_t), \pi^b(s_t))\bigg]\geq 0\quad i=1,2\nonumber
\end{align}

where $s_0\sim\bbrho$, $\pi^a$ and $\pi^b$ are the policies for the service and flow, respectively. We note that the expectation in \eqref{simulation_prob} is with respect to both the stochastic policies and the transition probability. For simulations, we choose $L=5$,  $\mathcal{A}=[0.2,0.4,0.6,0.8]$,  and $\mathcal{B}=[0.4,0.5,0.6,0.7]$ for all states besides the state $s=L$, Further, we select Slater variable $\varphi=0.2$, number of iteration $T=100000$, $\tdc_1=0.02$ and conservative variable $\kappa$ is selected as the statement of Theorem \ref{thm:occupancy_measure}. The initial distribution $\bbrho$ is set as uniform distribution. Moreover, the cost function is set to be $c(s,a,b)=-s+5$, 	the constraint function for the service is defined as $c^1(s,a,b)=-10a+3$, and the constraint function for the flow is $c^2(s,a,b)=-8(1-b)^2+1.2$. We run 200 independent simulations and collect the mean value and standard variance. In  Fig. \ref{fig:compare1} and \ref{fig:compare2}, we show the learning process of cumulative reward and constraint value for $\kappa=0$ and $\kappa>0$ respectively. Note that the y-axis in Fig. \ref{fig:compare1} and \ref{fig:compare2} are cumulative reward and constraint function defined in Eq. \eqref{simulation_prob}. It can be seen that when $\kappa>0$, the constraint values are strictly larger than 0, which matches the result in theory. Further, the rewards are similar for both $\kappa=0$ and $\kappa>0$, while the case where $\kappa>0$ helps to achieve zero constraint violation.

\section{Conclusion}
In this work, we considered the problem of learning optimal policies for infinite-horizon constrained Markov Decision Processes (CMDP) under finite state $\mathcal{S}$ and action $\ccalA$ spaces with $I$ number of constraints. This problem is also called as the constrained reinforcement learning (CRL) in the literature. To solve the problem in a model-free manner, we proposed a novel Conservative Stochastic Primal-Dual Algorithm (CSDPA) based upon the randomized primal-dual saddle point approach proposed in \cite{wang2020randomized}. We show that to achieve an $\epsilon$-optimal policy, it is sufficient to run the proposed Algorithm \ref{alg:spdgd} for $\Omega(\frac{I|\ccalS||\ccalA|\log(|\ccalS||\ccalA|)}{(1-\gamma)^2\varphi^2\eps^2})$ steps. Additionally, we proved that the proposed Algorithm \ref{alg:spdgd} does not violate any of the $I$ constraints which is unique to this work in the CRL literature. The idea is to consider a conservative version (controlled by parameter $\kappa$) of the original constraints and then a suitable choice of $\kappa$ enables us to make the constraint violation zero while still achieving the best sample complexity for the objective suboptimality.

\bibliography{aaai22}

\appendix
\onecolumn
\section*{Supplementary Material for 
\\``Achieving Zero Constraint Violation for Constrained Reinforcement Learning \\via Primal-Dual Approach"}

\section{Preliminaries}\label{sec:app_compare}
\begin{table*}[t]
	\centering
	\resizebox{0.8\textwidth}{!}{\begin{tabular}{|c|c|c|c|c|}
			\hline
			&   Algorithm                                           &                       Sample Complexity                                 & Constraint violation               & Generative Model \\
			\hline
			Model-Based &   OptDual-CMDP \cite{efroni2020exploration} \footnotemark[1]           & $\tdO\bigg(\frac{|\ccalS|^{2}|\ccalA|}{(1-\gamma)^{3}\epsilon^2}\bigg)$ &$\tdO(\epsilon)$   & {No} \\
			\hline
			&   OptPrimalDual-CMDP \cite{efroni2020exploration} \footnotemark[1]    & $\tdO\bigg(\frac{|\ccalS|^{2}|\ccalA|}{(1-\gamma)^{3}\epsilon^2}\bigg)$ &  $\tdO(\epsilon)$   & {No} \\
			\hline
			&   UC-CFH \cite{Kalagarla_Jain_Nuzzo_2021} \footnotemark[2]            & $\tdO\bigg(\frac{|\ccalS|^3|\ccalA|}{(1-\gamma)^3\epsilon^2}\bigg)$     & $\tdO(\epsilon)$   & {No} \\
			\hline
			&   CONRL\cite{brantley2020constrained}                 & $\tdO\bigg(\frac{|\ccalS|^2|\ccalA|}{(1-\gamma)^6\epsilon^2}\bigg)$     & $\tdO(\epsilon)$   & {No} \\
			\hline
			&   OptPess-PrimalDual \cite{liu2021learning}           & $\tdO\bigg(\frac{|\ccalS|^3|\ccalA|}{(1-\gamma)^4\epsilon^2}\bigg)$     & Zero & {No}\\
			\hline
			&   OPDOP \cite{ding2021provably}                       & $\tdO\bigg(\frac{|\ccalS|^2|\ccalA|}{(1-\gamma)^4\epsilon^2}\bigg)$     &$\tdO(\epsilon)$   & {No}\\
			\hline
			&   UCBVI-$\gamma$ \cite{he2021nearly}                  & $\tdO\bigg(\frac{|\ccalS||\ccalA|}{(1-\gamma)^3\epsilon^2}\bigg)$       & {N/A}                 & {No}\\
			\hline\hline
			Model-Free  &   NPG-PD \cite{ding2020natural} \footnotemark[3]                      & $\tdO\bigg(\frac{|\ccalS||\ccalA|}{(1-\gamma)^{5}\epsilon^2}\bigg)$     & $\tdO(\epsilon)$   & {Yes}\\
			\hline
			&   CRPO \cite{xu2021crpo} \footnotemark[4]                              & $\tdO\bigg(\frac{|\ccalS||\ccalA|}{(1-\gamma)^{7}\epsilon^4}\bigg)$     &  $\tdO(\epsilon)$   & {Yes}\\
			\hline
			&   PDSC \cite{chen2021primal} \footnotemark[5]                                 & $\tdO\bigg(\frac{1}{(1-\gamma)^{4}\epsilon^2}\bigg)$                 &    $\tdO(\epsilon)$                          & {Yes}\\
			\hline
			&   Triple-Q \cite{wei2021provably}                      & $\tdO\bigg(\frac{|\ccalS|^{2.5}|\ccalA|^{2.5}}{(1-\gamma)^{18.5}\epsilon^5}\bigg)$    & {Zero}  & {No} \\
			\hline
			&   Randomized Primal–Dual \cite{wang2020randomized}                  & $\tdO\bigg(\frac{|\ccalS||\ccalA|}{(1-\gamma)^4\epsilon^2}\bigg)$   & {N/A}                 & {Yes} \\
			\hline
			\rowcolor{LightCyan}       &   \textbf{CSPDA} (\textbf{This work}, Theorem \ref{thm:policy}) \footnotemark[6]& $\tdO\bigg(\frac{|\ccalS||\ccalA|}{(1-\gamma)^4\epsilon^2}\bigg)$       & \textbf{Zero}    & {\textbf{Yes}}\\
			\hline\hline
			Lower bound &    \cite{Lattimore2012lowerbound} and \cite{azar2013minimax}                    & $\tilde{\Omega}\bigg(\frac{|\ccalS||\ccalA|}{(1-\gamma)^3\epsilon^2}\bigg)$   & {N/A}                 & {N/A}\\
			\hline
	\end{tabular}}
	\caption{\scriptsize This table summarizes the different model-based and mode-free state of the art algorithms available in the literature for CMDPs. We note that the proposed algorithm in this work is able to achieve the best sample complexity among them all while achieving zero constraint violation as well. For the works considering different setting such as episodic setting, we provide a detailed method to convert the result to the form of sample complexity in infinite horizon setup in Appendix \ref{sec:app_convert}.
	}
	\label{table2}
\end{table*}

\footnotetext[1]{\cite{efroni2020exploration} used $\ccalN$, which is the maximum number of non-zero
	transition probabilities across the entire state-action pairs. We bound it by $\ccalS$. Moreover, a factor of $\sqrt{|\ccalA|}$ is missed in their result, which we believe is a typo in their work.}
\footnotetext[2]{\cite{Kalagarla_Jain_Nuzzo_2021} used $C$, which is the upper bound on the number of possible successor
	states for a state-action pair. We bound it by $\ccalS$.}
\footnotetext[3]{We use the result in Theorem 4 in \cite{ding2020natural}. Notice that in the Algorithm 2 of their paper, $\frac{|\ccalS||\ccalA|}{1-\gamma}$ samples are necessary for each outer loop.}

\footnotetext[4]{Notice that in line 4 of Algorithm 1 in \cite{xu2021crpo}, a inner loop with $K_{in}$ iteration is needed for policy evaluation and $K_{in}=\tdO(\frac{T}{(1-\gamma)|\ccalS||\ccalA|})$}
\footnotetext[5]{The dependence on $\ccalS,\ccalA$ is not clear in \cite{chen2021primal}. An estimation for the Q-function is needed in the algorithm. However, the authors didn't include analysis for the estimation.}
\footnotetext[6]{Notice that the value function defined in this paper is a normalized version. Thus, an extra $\frac{1}{(1-\gamma)^2}$ is needed for a fair comparison.}

\subsection{Explanation of comparison among references in Table \ref{table2}}\label{sec:app_convert}
\subsubsection*{Step 1: From Regret to PAC result}
Many references listed in the Table \ref{table2} are in the episodic setting and give the result in the form of regret, which is defined as
\begin{equation}\label{eq:regret_form}
    \sum_{k=1}^{K} V_{r,1}^*(s_1)-V_{r,1}^{\pi_k}(s_1)\leq f(H, |\ccalS|, |\ccalA|, T, \delta)\quad \textbf{with probability at least $1-\delta$}
\end{equation}
where $T=KH$. The following method provides a probably approximately correct (PAC) result from the regret. At the end of learning horizon $K$, a policy $\bbarpi$ can be defined as follow
\begin{equation}
	\bar{\pi}(s) = \begin {cases}
	\pi_1(s) & \text{ with probability } 1/K\\
	\cdots & \cdots\\
	\pi_k(s) & \text{ with probability } 1/K\\
	\cdots & \cdots\\
	\pi_K(s) & \text{ with probability } 1/K
	\end{cases}
\end{equation}
Note that $\bar{\pi}$ chooses the different policies $\pi^k$ for $k\in[K]$ uniformly at random. Thus, we know $\frac{1}{K}\sum_{k=1}^{K}V_{r,1}^{\pi_k}(s_1)=V_{r,1}^{\bbarpi}(s_1)$. Divide Eq. \eqref{eq:regret_form} by $K$ on both side, we have
\begin{equation}
    V_{r,1}^{*}(s_1)-V_{r,1}^{\bbarpi}(s_1)\leq \frac{f(H, |\ccalS|, |\ccalA|, T, \delta)}{K}
\end{equation}
If the function $f$ is sub-linear w.r.t. $T$, then for large enough $K$, we have $V_{r,1}^{*}(s_1)-V_{r,1}^{\bbarpi}(s_1)\leq \eps$ with probability at least $1-\delta$, which means that $\bbarpi$ is an $\eps$-optimal policy.

\subsubsection*{Step 2: From episodic setting to infinite horizon discounted setting}

As mentioned above, many references consider the problem in episodic setting. In order to make a comparison, it is necessary to have a fair conversion. Here, we use the method from \cite{jin2018qlearning}[footnote 3 in page 3]. Firstly, we check whether the MDP model in the given result assume a horizon dependent transition dynamics, i.e, whether $\bbP$ is a function of $h$. If so, then define $\ccalS'=\ccalS H$. If not, then define $\ccalS'=\ccalS$. This conversion is easy to understand and reasonable because an extra $H$ times state space is needed if transition dynamics is different for each $h$. After this step, we change $H$ to $\frac{1}{1-\gamma}$. This is because the infinite horizon discounted value function can be simulated by the following algorithm.

\begin{algorithm}[ht]
	\caption{Unbiased estimator for Value Function}
	\textbf{Input}: Initial distribution $\bbrho$. Discounted factor $\gamma$. Policy $\pi$\\
	\textbf{Output}: Value function $V_{r,1}^{\pi}$
	\begin{algorithmic}[1] 
		\STATE Sample $s_1\sim \bbrho$, $H\sim Geo(1-\gamma)$
		\FOR{Each state $s_1$ in $\ccalS$}
		\FOR{$h=1,2,...,H$}
		\STATE Take action $a_h\sim \pi(\cdot\vert s_h)$, observe next state $s_{h+1}$ and reward $r(s_h,a_h)$
		\ENDFOR
		\STATE $V_{r,1}^{\pi}(s_1)=\sum_{h=1}^{H}r(s_h,a_h)$
		\ENDFOR
		\end{algorithmic}
    \end{algorithm}

The sample horizon is taken from the geometry distribution with parameter $(1-\gamma)$ and thus the expected length of horizon is $\frac{1}{1-\gamma}$, which explains why it is fair to change $H$ to $\frac{1}{1-\gamma}$. Following these two steps, we convert the result in episodic setting into infinite horizon discounted setting.

\subsubsection*{Step3: From High Probability result to Expectation result}
After converting the result from episodic setting to infinite horizon discounted setting, we get an $\eps$-optimal result with probability at least $1-\delta$. However, the result in this paper is in the form of expectation. Thus, we can convert the result with the following method. Notice that the value function $V_r$ is bounded by $\frac{1}{1-\gamma}$, we have
\begin{equation}
    \mbE[V_{r}^*(s_1)-V_r^{\pi}(s_1)]\leq \eps*(1-\delta)+\delta * \frac{1}{1-\gamma}
\end{equation}
If $\delta<\eps(1-\gamma)$, then, we have $\mbE[V_{r}^*(s_1)-V_r^{\pi}(s_1)]\leq 2\eps$.

\subsubsection*{An example for UC-CFH in \cite{Kalagarla_Jain_Nuzzo_2021}}
In the UC-CFH algorithm, the author proposed an $\eps$-optimal result with at most $\tdO(\frac{|\ccalS||\ccalA|C^2H^2}{\eps^2}\log(\frac{1}{\delta}))$ episodes, where $C$ is the upper bound on the number of possible successor
states for a state-action pair. Thus, $C<|\ccalS|$ and the above equation can be bounded by $\tdO(\frac{|\ccalS|^3|\ccalA|H^2}{\eps^2}\log(\frac{1}{\delta}))$. Notice that this is already a PAC result and we begin converting it into infinite horizon discounted setting.

\begin{itemize}
    \item Firstly, we know $K=\tdO(\frac{|\ccalS|^3|\ccalA|H^2}{\eps^2}\log(\frac{1}{\delta}))$ and thus the total sample complexity is $KH = \tdO(\frac{|\ccalS|^3|\ccalA|H^3}{\eps^2}\log(\frac{1}{\delta}))$. Notice that UC-CFH algorithm doesn't assume horizon dependent transition dynamics (They assume in the model, however, not in the algorithm and theorem). Thus, by changing $H$ to $\frac{1}{1-\gamma}$, we have sample complexity $ \tdO(\frac{|\ccalS|^3|\ccalA|}{(1-\gamma)^3\eps^2}\log(\frac{1}{\delta}))$.
    \item Secondly, change $\delta$ to $\eps(1-\gamma)$, we get the sample complexity in the form of expectation, which means with $\tdO(\frac{|\ccalS|^3|\ccalA|}{(1-\gamma)^3\eps^2})$ sample, we have
    \begin{equation}
        \mbE[V_1^*(s_1)-V_1^{\pi_k}(s_1)]\leq \eps
    \end{equation}
\end{itemize}

\section{Notations}
For the purpose of analysis in the appendix, we have used the shorthand notation $\bblambda_{sa}$ for $\bblambda(s,a)$, $\bbr_{sa}$ for $\bbr(s,a)$, and $\bbg^i_{sa}$ for $\bbg^i{(s,a)}$.

\section{Proofs for Section \ref{sec:duality_gap}}
\subsection{Proof of Lemma \ref{lem:bounded_dual_u}}\label{proof_lemma_1}
\begin{proof}
   \textbf{Bound on $\Vert \bbu_{\kappa}^*\Vert_1$:} Let us denote the optimal value of optimization problem in \eqref{eq:formulation_conservative} as $p^*_{\kappa}$ and write the corresponding dual problem as
    \begin{equation}
    	\begin{aligned}
    	\ccalD_{\kappa}(\bbu,\bbv)&:=\max_{\bblambda\geq\bbzero}\ccalL(\bblambda,\bbu,\bbv)=\max_{\bblambda\geq\bbzero}\left<\bblambda,\bbr\right>+\left<\bbu,\bbG^T\bblambda-\kappa\bbone\right>+(1-\gamma)\left<\rho,\bbv\right>+\sum_{a}\bblambda_a^T(\gamma\bbP_a-\bbI)\bbv.
    	\end{aligned}
    \end{equation}
    The optimal dual variables are given by  
    \begin{align}
        (\bbu_{\kappa}^*,\bbv_{\kappa}^*):=\arg\min_{\bbu\geq\bbzero, \bbv}\ccalD_{\kappa}(\bbu,\bbv),
    \end{align}
    and let us denote the optimal dual value by $d_{\kappa}^*=\ccalD_{\kappa}(\bbu_{\kappa}^*,\bbv_{\kappa}^*)$.
    We note that the problem in \eqref{eq:formulation_conservative} is a LP and strong duality holds, i.e $p_{\kappa}^*=d_{\kappa}^*$. To proceed, let us consider a constant $C$ and define a set $\mathcal{C}:=\{(\bbu,\bbv)\geq\bbzero\vert \ccalD_\kappa(\bbu,\bbv)\leq C\}$. For any $(\bbu,\bbv)\in\mathcal{C}$ and a feasible $\hat\bblambda$ which satisfies Assumption \ref{ass:Slater}, we could write 
        \begin{align}\label{eq:bound_dual_u_general}
        C\geq \ccalD_\kappa(\bbu,\bbv)\overset{(a)}\geq&\ccalL(\hat{\bblambda},\bbu,\bbv)\nonumber
        \\
        =&\left<\hat{\bblambda},\bbr\right>+\left<\bbu,\bbG^T\hat{\bblambda}-\kappa\bbone\right>+(1-\gamma)\left<\rho,\bbv\right>+\sum_{a}\hat{\bblambda}_a^T(\gamma\bbP_a-\bbI)\bbv
        \nonumber
        \\
        \overset{(b)}\geq&\left<\hat{\bblambda},\bbr\right>+\left<\bbu,\frac{\varphi\bbone}{2}\right>
        \nonumber
        \\
        =&\left<\hat{\bblambda},\bbr\right>+\frac{\varphi}{2}\Vert \bbu\Vert_1,
        \end{align}
where step (a) holds by the definition of dual function and step (b) is true by Assumption \ref{ass:Slater} and $\kappa\leq\frac{\varphi}{2}$. 
\begin{equation}
	\begin{aligned}
		&\left<\hat{\bblambda}, \bbg^i\right>\geq \kappa \quad\forall i\in[I]\\
		\text{and}\quad & \sum_{a}(\bbI-\gamma \bbP_a^T)\hat{\bblambda}_a=(1-\gamma)\bbrho
	\end{aligned}
\end{equation}
    From weak duality, we have
    \begin{equation}\label{eq:weak_duality}
        D_{\kappa}(\bbu,\bbv)\geq d_{\kappa}^*\geq p_{\kappa}^*=\left<\bblambda^*,\bbr\right> 
    \end{equation}
    Now let $C=\left<\bblambda^*,\bbr\right>$, all inequalities in Eq. \eqref{eq:weak_duality} become equality for $(\bbu,\bbv)\in\{(\bbu,\bbv)\geq\bbzero\vert \ccalD_\kappa(\bbu,\bbv)\leq \left<\bblambda^*,\bbr\right>\}$. Thus, this set is the optimal dual variable set. We set $C=\left<\bblambda^*,\bbr\right>$ and rearrange the Eq. \eqref{eq:bound_dual_u_general} to obtain
    \begin{equation}
        \Vert \bbu_{\kappa}^*\Vert_1\leq \frac{2[\left<\bblambda^*,\bbr\right>-\left<\hat{\bblambda},\bbr\right>]}{\varphi}\leq \frac{2}{\varphi}
    \end{equation}
    where the last step holds because $0\leq \left<\bblambda,\bbr\right>\leq 1$ for any $\bblambda\in \Lambda$ because $\Lambda$ is a probability simplex.

 \textbf{Bound on $\Vert \bbv_{\kappa}^*\Vert_\infty$:} 	To solve the linear program in \eqref{eq:formulation_conservative}, the KKT conditions should be sufficient and necessary, which can be written as
	\begin{subequations}
		\begin{align}
		&\nabla_{\bblambda}\ccalL(\bblambda_{\kappa}^*,\bbu_{\kappa}^*,\bbv_{\kappa}^*)=0 \label{eq:KKTa}\\
		&\left<\bblambda_{\kappa}^*,\bbg^i\right>\geq\kappa\quad \forall i\in[I] \label{eq:KKTb}\\
		&\sum_{a}(\bbI-\gamma \bbP_a^T)\bblambda_{\kappa,a}^*=(1-\gamma)\bbrho \label{eq:KKTc}\\
		&\left<\bbu_{\kappa}^*,\bbG^T\bblambda_{\kappa}^*-\kappa\bbone\right>=0 \label{eq:KKTd}\\
		&\bbu_{\kappa}^*\geq\bbzero \label{eq:KKTe}
		\end{align}
	\end{subequations}
	By Eq. \eqref{eq:KKTa}, we have for any state-action pair $(s,a)$
	\begin{equation}\label{eq:bound_v_scalar}
	\bbr_{sa}+\sum_{i\in[I]}u_{\kappa,i}^*\bbg^i_{sa}-(\bbe_s-\gamma\bbP_{as})^T\bbv_{\kappa}^*=0,
	\end{equation}
	where $u_{\kappa,i}^*$ is the $i^{th}$ elemnt of vector $u_{\kappa}^*$, $\bbP_{as}$ is a column vector and $P_{as}(s')=P(s'|a,s)$. Given a fixed action $\bbara$, denote $\tbr:=[\bbr_{1\bbara},\bbr_{2\bbara},\cdots,\bbr_{|\ccalS\bbara|}]^T$, $\tbg^i:=[\bbg^i_{1\bbara},\bbg^i_{2\bbara},\cdots,\bbg^i_{|\ccalS\bbara|}]^T$ and $\tbP:=[P_{\bbara,1},\cdots,P_{\bbara,|\ccalS|}]\in\mbR^{|\ccalS|\times|\ccalS|}$. By Eq. \eqref{eq:bound_v_scalar}, we have
	\begin{equation}
	    (\bbI-\gamma\tbP^T)\bbv_{\kappa}^*=\tbr+\sum_{i\in[I]}u_{\kappa,i}^*\tbg_i
	\end{equation}
	As a result, we have
	\begin{equation}
	    \begin{aligned}
	    1+\frac{2}{\varphi}&\overset{(a)}\geq 1+\Vert\bbu_{\kappa}^*\Vert_1\overset{(b)}\geq \Vert \tbr+\sum_{i\in[I]}u_{\kappa,i}^*\tbg_i\Vert_{\infty}=\Vert (\bbI-\gamma\tbP^T)\bbv_{\kappa}^*\Vert_\infty\\
	    &\overset{(c)}\geq \Vert\bbv_{\kappa}^*\Vert_{\infty}-\Vert \gamma\tbP^T\bbv_{\kappa}^*\Vert_{\infty}\overset{(d)}\geq(1-\gamma)\Vert\bbv_{\kappa}^*\Vert_{\infty},
	    \end{aligned}
	\end{equation}
	where the step (a) holds by the Lemma \ref{lem:bounded_dual_u},  step (b) holds by the definition of $r,g_i$, step (c) comes from the triangle inequality,  and step (d) is true because each row in $\tbP^T$ adds up to 1. Finally, we have the bound $\Vert \bbv_{\kappa}^*\Vert_\infty\leq \frac{1}{1-\gamma}+\frac{2}{(1-\gamma)\varphi}$.
\end{proof}
\subsection{Proof of Lemma \ref{lem:dual_gap_bound}}\label{lemma_3proof}
\begin{proof}
	Consider the Lagrangian in \eqref{eq_Lagrange} and note that it is convex w.r.t $\bbu$ as well as $\bbv$. w.r.t The gradient of the Lagrange function $\bbu$ and $\bbv$ are given by
	\begin{equation}\label{eq:true_augmented_gradient}
		\begin{aligned}
		\nabla_{\bbu}\ccalL(\bblambda,\bbu,\bbv)&=\bbG^T\bblambda-\kappa\bbone,\\
		\nabla_{\bbv}\ccalL(\bblambda,\bbu,\bbv)&=(1-\gamma)\bbrho+\sum_a(\gamma\bbP_a^T-\bbI)\bblambda_a.
		\end{aligned}
	\end{equation}
	It is obvious that $\nabla^2_{\bbu}\ccalL(\bblambda,\bbu,\bbv)=\nabla_{\bbu,\bbv}\ccalL(\bblambda,\bbu,\bbv)=\nabla_{\bbv,\bbu}\ccalL(\bblambda,\bbu,\bbv)=\nabla^2_{\bbv}\ccalL(\bblambda,\bbu,\bbv)=\bbzero$, which means that the Hessian matrix $\nabla_{\bbw}\ccalL(\bblambda,\bbu,\bbv)$ is a zero matrix. Thus, Lagrange function is convex w.r.t $\bbw$. Then, let us define $\bbw= [\bbu^T, \bbv^T]^T$, $\bar\bbw=\frac{1}{T}\sum_{t=1}^{T}\bar\bbw_t$, and decompose the duality gap as 
	\begin{equation}\label{eq:gap_decompose}
		\begin{aligned}
		\ccalL(\barbu,\barbv,\bblambda_{\kappa}^*)-\ccalL(\bbu,\bbv,\barblambda)&=\ccalL(\barbw,\bblambda_{\kappa}^*)-\ccalL(\bbw,\barblambda)\\
		&\overset{(a)}\leq \frac{1}{T}\sum_{t=1}^{T}\big[\ccalL(\bbw^t,\bblambda_{\kappa}^*)-\ccalL(\bbw,\bblambda^t)\big]
		\\
		&=\frac{1}{T}\sum_{t=1}^{T}\big[\ccalL(\bbw^t,\bblambda_{\kappa}^*)-\ccalL(\bbw^t,\bblambda^t)+\ccalL(\bbw^t,\bblambda^t)-\ccalL(\bbw,\bblambda^t)\big]
		\\
		&\overset{(b)}\leq \frac{1}{T}\sum_{t=1}^{T}\big[\left<\nabla_\bbw\ccalL(\bbw^t,\bblambda^t),\bblambda_{\kappa}^*-\bblambda^t\right>+\left<\nabla_{\bblambda}\ccalL(\bbw^t,\bblambda^t),\bbw^t-\bbw\right>\big],
		\end{aligned}
	\end{equation}
	where step (a) holds by Jensen inequality and the step (b) utilizes the convexity of $\ccalL(\cdot, \bblambda)$ and concavity of $\ccalL(\bbw,\cdot)$. 
\end{proof}

\subsection{Proof of Theorem \ref{thm:duality_gap}} \label{proof_theorem_1}
\begin{proof}
We collect the dual variables $\bbu$ and $\bbv$ in one variable $\bbw$ as defined in Lemma \ref{lem:dual_gap_bound} for the ease of analysis. The next two Lemmas provide the bound on the terms I and II in Eq. \eqref{eq:duality_gap_decompostion}. 
\begin{lemma}\label{lem:bound_primal}
	Let the iterate sequence $\{\bblambda^t\}$ be updated as mentioned in  the updates \eqref{eq:update_lambda1} and \eqref{eq:update_lambda2} of Algorithm \ref{alg:spdgd}, then for any $t$ it holds that
	\begin{align}\label{eq:bound_primal}
	&\left<\nabla_{\bblambda} \ccalL(\bbw^t,\bblambda^t),\bblambda-\bblambda^t\right>
	\nonumber\\
	&\quad\leq \frac{1}{\beta}\big[KL(\bblambda||\bblambda^t)-KL(\bblambda||\bblambda^{t+1})\big]\!\!+\!\!\frac{\beta}{2}\sum_{s,a}\lambda_{sa}^t(\Delta_{sa}^t)^2\nonumber
	\\
	&\qquad+\left<\hat{\nabla}_{\bblambda}\ccalL(\bbw^t,\bblambda^t)-\nabla_{\bblambda}\ccalL(\bbw^t,\bblambda^t),\bblambda^t-\bblambda\right>.
	\end{align}
\end{lemma}

\begin{lemma}\label{lem:bound_dual}
   Define $\ccalW=\ccalU\times\ccalV$ and consider the iterate sequence $\{\bbw^t\}$ updated according to the rule Eq. \eqref{eq:update_u} and \eqref{eq:update_v} in Algorithm \ref{alg:spdgd}. For any $t$, it holds that
    	\begin{align}\label{eq:bound_dual}
    		&\left<\nabla_\bbw\ccalL(\bbw^t,\bblambda^t),\bbw^t-\bbw\right>
    		\\
    		&\quad\leq \frac{1}{2\alpha}\bigg[\Vert\bbw^t\!-\!\bbw\Vert^2\!-\!\Vert\bbw^{t+1}\!-\!\bbw\Vert^2\!+\!\alpha^2\Vert\hat{\nabla}_{\bbw}\ccalL(\bbw^t,\bblambda^t)\Vert^2\nonumber\\
    		&\quad\quad\qquad+2\alpha\left<\nabla_\bbw\ccalL(\bbw,\bblambda)-\hat{\nabla}_\bbw\ccalL(\bbw,\bblambda),\bbw^t-\bbw\right>\bigg].\nonumber
    	\end{align}
\end{lemma}
Next, utilizing the results of Lemma \ref{lem:bound_primal} and \ref{lem:bound_dual} (see proofs in Appendix \ref{lemma_4} and \ref{lemma_5}) into Lemma \ref{lem:dual_gap_bound}, we prove the main result in Theorem \ref{thm:duality_gap}, which establishes the final bound on the duality gap as follows. 
Let $\bblambda=\bblambda_{\kappa}^*$ in Eq. \eqref{eq:bound_primal} and $(\bbu^\dagger,\bbv^\dagger):=\arg\min_{\bbu,\bbv}\ccalL(\bbu,\bbv,\barblambda)$ in Eq. \eqref{eq:bound_dual}. Then, sum up Eq. \eqref{eq:bound_primal} and \eqref{eq:bound_dual} from $t=1$ to $T$ , we have
	\begin{equation}
		\begin{aligned}
		\frac{1}{T}\sum_{t=1}^{T}\bigg[&\left<\nabla_{\bblambda} \ccalL(\bbw^t,\bblambda^t),\bblambda_{\kappa}^*-\bblambda^t\right>+\left<\nabla_\bbw\ccalL(\bbw^t,\bblambda^t),\bbw^t-\bbw^{\dagger}\right>\bigg]\\
		&\leq\underbrace{\frac{KL(\bblambda_{\kappa}^*||\bblambda^1)}{T\beta}}_{T_1}+\underbrace{\frac{\beta}{2T}\sum_{t=1}^{T}\sum_{s,a}\lambda_{sa}^t(\Delta_{sa}^t)^2}_{T_2}+\underbrace{\frac{1}{T}\sum_{t=1}^{T}\left<\hat{\nabla}_{\bblambda}\ccalL(\bbw^t,\bblambda^t)-\nabla_{\bblambda}\ccalL(\bbw^t,\bblambda^t),\bblambda^t-\bblambda_{\kappa}^*\right>}_{T_3}\\
		&\quad+\underbrace{\frac{1}{2T\alpha}\Vert\bbw^1-\bbw^{\dagger}\Vert^2}_{T_4}+\underbrace{\frac{\alpha}{2T}\sum_{t=1}^{T}\Vert\hat{\nabla}_{\bbw}\ccalL(\bbw^t,\bblambda^t)\Vert^2}_{T_5}+\underbrace{\frac{1}{T}\sum_{t=1}^{T}\left<\nabla_\bbw\ccalL(\bbw^t,\bblambda^t)-\hat{\nabla}_\bbw\ccalL(\bbw^t,\bblambda^t),\bbw^t-\bbw^{\dagger}\right>}_{T_6}
		\end{aligned}
	\end{equation}
	Combine the above result with the statement of Lemma. \ref{lem:dual_gap_bound} to write 
	\begin{equation}\label{here2}
		\mbE[\ccalL(\barbu,\barbv,\bblambda_{\kappa}^*)-\ccalL(\bbu^{\dagger},\bbv^{\dagger},\barblambda)]\leq 	\sum_{j=1}^{6}\mbE[T_j].
	\end{equation}
	We derive an upper bound on the right hand side of \eqref{here2} in Appendix \ref{first_T_1}-\ref{first_T_6}. Following the results in Appendix \ref{first_T_1}-\ref{first_T_6}, we have
	\begin{equation}
		\begin{aligned}
		&\mbE[T_1]\leq \frac{\log(|\ccalS||\ccalA|)}{T\beta},\quad 
		\mbE[T_2]\leq \frac{1024\beta|\ccalS||\ccalA|}{(1-\gamma)^2\varphi^2}\quad \mbE[T_3]=0, \\
		&\mbE[T_4]\leq \frac{104|\ccalS|}{T\alpha(1-\gamma)^2\varphi^2},\quad 
		\mbE[T_5]\leq \frac{37}{2}I\alpha,\quad 
		\mbE[T_6]\leq \frac{48\sqrt{|\ccalS|I}}{\sqrt{T}(1-\gamma)\varphi}.
		\end{aligned}
	\end{equation}
	Let $\beta=(1-\gamma)\varphi\sqrt{\frac{\log(|\ccalS||\ccalA|)}{T|\ccalS||\ccalA|}}$ and $\alpha=\frac{\sqrt{|\ccalS }|}{(1-\gamma)\varphi\sqrt{TI}}$, the final bound for duality gap could be written as 
	\begin{equation}
		\begin{aligned}
		\mbE[\ccalL(\barbu,\barbv,\bblambda_{\kappa}^*)-\ccalL(\bbu^{\dagger},\bbv^{\dagger},\barblambda)]&\leq \frac{\sqrt{|\ccalS||\ccalA|\log(|\ccalS||\ccalA|)}}{\sqrt{T}(1-\gamma)\varphi}+\frac{1024\sqrt{|\ccalS||\ccalA|\log(|\ccalS||\ccalA|)}}{\sqrt{T}(1-\gamma)\varphi}\\
		&\quad+\frac{104\sqrt{|\ccalS|I}}{\sqrt{T}(1-\gamma)\varphi}+\frac{37\sqrt{|\ccalS|I}}{2\sqrt{T}(1-\gamma)\varphi}+\frac{48\sqrt{|\ccalS|I}}{\sqrt{T}(1-\gamma)\varphi}\\
		&\leq \ccalO\bigg(\sqrt{\frac{I|\ccalS||\ccalA|\log(|\ccalS||\ccalA|)}{T}}\cdot \frac{1}{(1-\gamma)\varphi}\bigg),
	\end{aligned}
	\end{equation}
	which is as stated in the statement of Theorem  \ref{thm:duality_gap}.
\end{proof}

\subsection{Proof of Lemma \ref{lem:bound_primal}} \label{lemma_4}
The Proof of Lemma \ref{lem:bound_primal} in this work follows similar logic to  \cite[Lemma C.2]{Junyu}. The main difference lies in the selection of shift parameters $M$ and we provide the proof here for completeness. 
\begin{proof}
	Let us defined $\Delta_{sa}$ as the $(s,a)$-th component of $\hat{\nabla}_{\bblambda}\ccalL(\bblambda^t,\bbu^t,\bbv^t)$. Consider the update in Eq. \eqref{eq:update_lambda1} and note that the problem is separable for each component of $\bblambda$ and could be solved in closed form as follows. 
	    \begin{align}
    	\max_{\bblambda\in\Lambda}&\left<\hat{\nabla}_{\bblambda}\ccalL(\bblambda^t,\bbu^t,\bbv^t),\bblambda-\bblambda^t\right>-\frac{1}{\beta}KL(\bblambda\Vert\bblambda^t)\nonumber
    	\\
    	&=\max_{\bblambda\in\Lambda}\quad\left<\hat{\nabla}_{\bblambda}\ccalL(\bblambda^t,\bbu^t,\bbv^t),-\bblambda^t\right>+\max_{\bblambda}\left\{\sum_{s,a}\Delta_{sa}^t\lambda_{sa}-\frac{1}{\beta}\sum_{s,a}\lambda_{sa}\log\left(\frac{\lambda_{sa}}{\lambda_{sa}^t}\right)\right\}\nonumber
    	\\
    	&=\max_{\bblambda\in\Lambda}\quad\sum_{s,a}\max_{\lambda_{sa}}\left\{\lambda_{sa}\bigg[\Delta_{sa}^t-\frac{1}{\beta}\log\left(\frac{\lambda_{sa}}{\lambda_{sa}^t}\right)\bigg]\right\},\label{here}
    	\end{align}
	where we drop the terms which does not depend upon the variable $\bblambda$ and $\Lambda$ denotes the set of probability distributions. Next, we solve the unconstrained maximization in \eqref{here} by differentiating and equating it to zero as follows 
	\begin{equation}
	    \frac{d}{d\lambda_{sa}}\bigg(\lambda_{sa}\bigg[\Delta_{sa}^t-\frac{1}{\beta}\log\left(\frac{\lambda_{sa}}{\lambda_{sa}^t}\right)\bigg]\bigg)\Bigg|_{\lambda_{sa}=\lambda_{sa}^{t+\frac{1}{2}}}=\Delta_{sa}^t-\frac{1}{\beta}\log\left(\frac{\lambda_{sa}^{t+\frac{1}{2}}}{\lambda_{sa}^t}\right)+\frac{1}{\beta}=0.
	\end{equation}
	After rearranging the terms, we obtain
	\begin{equation}
        \lambda_{sa}^{t+\frac{1}{2}}=\lambda_{sa}^t\exp(\beta\Delta_{sa}^t+1).
	\end{equation}
	Now, we project back the solution on to the set of valid probability distribution and obtain the update as 
	\begin{equation}\label{dual_update}
	    \lambda_{sa}^{t+1}=\frac{\lambda_{sa}^t\cdot\exp(\beta\Delta_{sa}^t)}{\sum_{s',a'}\lambda_{s'a'}^t\cdot\exp(\beta\Delta_{s'a'}^t)},
	\end{equation}
	where we note that $ \lambda_{sa}^{t+1}\in\Lambda$. 
	Next, we analyze the one step KL divergence of $\bblambda^{t+1}$ to any $\bblambda$  as 
		\begin{align}
		KL(\bblambda||\bblambda^t)-KL(\bblambda||\bblambda^{t+1})=&\sum_{s,a}\lambda_{sa}\log\left(\frac{\lambda_{sa}}{\lambda_{sa}^t}\right)-\sum_{s,a}\lambda_{sa}\log\left(\frac{\lambda_{sa}}{\lambda_{sa}^{t+1}}\right)\nonumber
		\\
		=&\sum_{s,a}\lambda_{sa}\log\left(\frac{\lambda_{sa}^{t+1}}{\lambda_{sa}^t}\right).
		\end{align}
		Next, we substitute the definition of $\lambda^{t+1}_{sa}$ to obtain
				\begin{align}\label{eq:bound_primal1}
		KL(\bblambda||\bblambda^t)-KL(\bblambda||\bblambda^{t+1})
		=&\sum_{s,a}\lambda_{sa}\bigg[\beta\Delta_{sa}^t-\log\left(\sum_{s',a'}\lambda_{s'a'}^t\cdot\exp(\beta\Delta_{s'a'}^{t})\right)\bigg]\nonumber\\
		=&\beta\left<\bblambda,\hat{\nabla}_{\bblambda}\ccalL(\bblambda^t,\bbu^t,\bbv^t)\right>-\log\left(\sum_{s',a'}\lambda_{s'a'}^t\cdot\exp(\beta\Delta_{s'a'}^{t})\right),
		\end{align}
		where we utilize the fact that $\sum_{s,a}\lambda_{sa}=1$. To proceed next, recall that we have $\Delta_{sa}=\frac{r_{sa}+\gamma v_{s'}-v_s+\sum_{i\in[I]}u_ig_{i,sa}-M}{\zeta_{sa}}$. We note that
	\begin{equation}
	    |v_{s}|\leq 2\left[\frac{1}{1-\gamma}+\frac{2}{(1-\gamma)\varphi}\right], \quad \quad |r_{sa}+\gamma v_{s'}|\leq 2\left[\frac{1}{1-\gamma}+\frac{2}{(1-\gamma)\varphi}\right], \quad \text{and} \quad \bigg|\sum_{i\in[I]}u_ig_{sa}\bigg|\leq \frac{4}{\varphi}. 
	\end{equation}
	Hence, with the selection  $M=4\left[\frac{1}{\varphi}+\frac{1}{1-\gamma}+\frac{2}{(1-\gamma)\varphi}\right]$, we can conclude that $\Delta_{sa}\leq 0$. Since $\exp(x)\leq (1+x+\frac{x^2}{2})$ for $x\leq 0$, we can upper bound the second term on the right hand side of \eqref{eq:bound_primal1} as
	\begin{equation}\label{eq:bound_primal2}
	    \begin{aligned}
	    \log\left(\sum_{s',a'}\lambda_{s'a'}^t\cdot\exp(\beta\Delta_{s'a'}^{t})\right)&\leq \log\bigg(\sum_{s',a'}\lambda^t_{s'a'}\cdot(1+\beta\Delta_{s'a'}^t+\frac{\beta^2}{2}(\Delta^t_{s'a'})^2)\bigg)\\
	    &=\log\bigg(1+\beta\sum_{s',a'}\lambda^t_{s'a'}\Delta^t_{s'a'}+\frac{\beta^2}{2}\sum_{s',a'}\lambda^t_{s'a'}(\Delta^t_{s'a'})^2\bigg)\\
	    &=\log\bigg(1+\beta\left<\hat{\nabla}_{\bblambda}\ccalL(\bbu^t,\bbv^t,\bblambda^t),\bblambda^t\right>+\frac{\beta^2}{2}\sum_{s',a'}\lambda^t_{s'a'}(\Delta^t_{s'a'})^2\bigg)\\
	    &\leq \beta\left<\hat{\nabla}_{\bblambda}\ccalL(\bbu^t,\bbv^t,\bblambda^t),\bblambda^t\right>+\frac{\beta^2}{2}\sum_{s',a'}\lambda^t_{s'a'}(\Delta^t_{s'a'})^2,
	    \end{aligned}
	\end{equation}
	where the last inequality holds by $\log(1+x)\leq x$ for all $x>-1$. By combining Eq. \eqref{eq:bound_primal1} and \eqref{eq:bound_primal2}, we obtain
	\begin{equation}
	    \begin{aligned}
	     KL(\bblambda||\bblambda^t)-KL(\bblambda||\bblambda^{t+1})&\geq \beta\left<\bblambda,\hat{\nabla}_{\bblambda}\ccalL(\bblambda^t,\bbu^t,\bbv^t)\right>-\beta\left<\hat{\nabla}_{\bblambda}\ccalL(\bbu^t,\bbv^t,\bblambda^t),\bblambda^t\right>-\frac{\beta^2}{2}\sum_{s',a'}\lambda^t_{s'a'}(\Delta^t_{s'a'})^2.
	    \end{aligned}
	\end{equation}
	Rearrange the items and divide both sides by $\beta$, to obtain
	\begin{equation}
	    0\leq \frac{1}{\beta}[ KL(\bblambda||\bblambda^t)-KL(\bblambda||\bblambda^{t+1})]+\left<\hat{\nabla}_{\bblambda}\ccalL(\bblambda^t,\bbu^t,\bbv^t),\bblambda^t-\bblambda\right>+\frac{\beta}{2}\sum_{s',a'}\lambda^t_{s'a'}(\Delta^t_{s'a'})^2.
	\end{equation}
	Add $\left<\nabla_{\bblambda}\ccalL(\bblambda^t,\bbu^t,\bbv^t),\bblambda-\bblambda^t\right>$ on both side to get the desired result.
\end{proof}
\subsection{Proof of Lemma \ref{lem:bound_dual}}\label{lemma_5}
\begin{proof}
	We can combine  the update rule in Eq. \eqref{eq:update_u}-\eqref{eq:update_v} to obtain an update for $\bbw\in\mathcal{W}:=\mathcal{U}\times\mathcal{V}$. For any $\bbw\in\ccalW$, it holds that
	\begin{equation}\label{eq:bound_part1}
	\begin{aligned}
	\Vert \bbw^{t+1}-\bbw\Vert^2&=\Vert \Pi_\ccalW(\bbw^t-\alpha\hat{\nabla}_{\bbw}\ccalL(\bbw^t,\bblambda^t))-\bbw\Vert^2\\
	&\leq \Vert\bbw^t-\alpha\hat{\nabla}_{\bbw}\ccalL(\bbw^t,\bblambda^t)-\bbw\Vert^2\\
	&=\Vert\bbw^t-\bbw\Vert^2+\alpha^2\Vert\hat{\nabla}_{\bbw}\ccalL(\bbw^t,\bblambda^t)\Vert^2-2\alpha\left<\hat{\nabla}_\bbw\ccalL(\bbw^t,\bblambda^t),\bbw^t-\bbw\right>\\
	&=\Vert\bbw^t-\bbw\Vert^2+\alpha^2\Vert\hat{\nabla}_{\bbw}\ccalL(\bbw^t,\bblambda^t)\Vert^2-2\alpha\left<\hat{\nabla}_\bbw\ccalL(\bbw^t,\bblambda^t)-\nabla_\bbw\ccalL(\bbw,\bblambda)+\nabla_\bbw\ccalL(\bbw,\bblambda),\bbw^t-\bbw\right>, 
	\end{aligned}
	\end{equation}
	where the first inequality holds by the non-expansiveness of the Projection operator. The following equalities holds by expanding the squares and by adding subtracting the term $2\alpha\left<\nabla_\bbw\ccalL(\bbw,\bblambda),\bbw^t-\bbw\right>$. After rearranging the terms in  the above expression, we obtain
	\begin{align}
	2\alpha\left<\nabla_\bbw\ccalL(\bbw,\bblambda),\bbw^t-\bbw\right>\leq & \Vert\bbw^t-\bbw\Vert^2-\Vert\bbw^{t+1}-\bbw\Vert^2+\alpha^2\Vert\hat{\nabla}_\bbw\ccalL(\bbw^t,\bblambda^t)\Vert^2\nonumber
	\\
	&-2\alpha\left<\hat{\nabla}_\bbw\ccalL(\bbw^t,\bblambda^t)-\nabla_\bbw\ccalL(\bbw^t,\bblambda^t),\bbw^t-\bbw\right>. 
	\end{align}
	Next, divide the both sides by $2\alpha>0$ to obtain  the statement of Lemma \ref{lem:bound_dual}.
\end{proof}

\subsection{Upper bound for $\mathbb{E}[T_1]$} \label{first_T_1}
\begin{equation}
	\begin{aligned}
		\mbE[T_1]&=\frac{KL(\bblambda_{\kappa}^*||\bblambda^1)}{T\beta}=\frac{1}{T\beta}\sum_{s,a}\lambda_{\kappa,sa}^*\log\bigg(\frac{\lambda_{\kappa,sa}^*}{\lambda_{sa}^1}\bigg)=\frac{1}{T\beta}\sum_{s,a}\lambda_{\kappa,sa}^*[\log\lambda_{\kappa,sa}^*-\log\lambda_{sa}^1]\\
		&\leq \frac{1}{T\beta}\sum_{s,a}\lambda_{\kappa,sa}^*\log(|\ccalS||\ccalA|)=\frac{\log(|\ccalS||\ccalA|)}{T\beta}.
	\end{aligned}
\end{equation}
\subsection{Upper bound for $\mathbb{E}[T_2]$}
For any fixed $\bbu^t, \bbv^t,\bblambda^t$, we have
\begin{equation}
    \begin{aligned}
    \mbE[\sum_{s,a}\lambda_{sa}^t(\Delta_{sa}^t)^2|\bbu^t,\bbv^t,\bblambda^t]&=\mbE_{s_t,a_t}\bigg[\sum_{s,a}\lambda_{sa}^t\bigg(\frac{r_{sa}+\gamma v_{s'}-v_s+\sum_{i}u_ig_{i,sa}-M}{\zeta_{sa}^t}\cdot\bbone_{(s,a)=(s_t,a_t)}\bigg)^2\bigg]\\
    &=\mbE_{s_t,a_t}\bigg[\lambda_{s_ta_t}^t\bigg(\frac{r_{s_ta_t}+\gamma v_{s_t'}-v_{s_t}+\sum_{i}u_ig_{i,s_ta_t}-M}{\zeta_{s_ta_t}^t}\bigg)^2\bigg]\\
    &=\sum_{s_t,a_t}\bigg[\lambda_{s_ta_t}^t\zeta_{s_ta_t}^t\bigg(\frac{r_{s_ta_t}+\gamma v_{s_t'}-v_{s_t}+\sum_{i}u_ig_{i,s_ta_t}-M}{\zeta_{s_ta_t}^t}\bigg)^2\bigg]\\
    &=\sum_{s_t,a_t}\frac{\lambda_{s_ta_t}^t\bigg(r_{s_ta_t}+\gamma v_{s_t'}-v_{s_t}+\sum_{i}u_ig_{i,s_ta_t}-M\bigg)^2}{(1-\delta)\lambda_{s_ta_t}^t+\frac{\delta}{|\ccalS||\ccalA|}}\\
    &\leq\sum_{s_t,a_t}\frac{\lambda_{s_ta_t}^t\bigg(r_{s_ta_t}+\gamma v_{s_t'}-v_{s_t}+\sum_{i}u_ig_{i,s_ta_t}-M\bigg)^2}{(1-\delta)\lambda_{s_ta_t}^t}\\
    &\leq \frac{64|\ccalS||\ccalA|[\frac{1}{\varphi}+\frac{1}{1-\gamma}+\frac{2}{(1-\gamma)\varphi}]^2}{1-\delta}\\
    &\leq\frac{1024|\ccalS||\ccalA|}{(1-\delta)(1-\gamma)^2\varphi^2}.
    \end{aligned}
\end{equation}
Next, let us write down the term $\mbE[T_2]$ as
\begin{equation}
	\begin{aligned}
	\mbE[T_2]=\mbE[\frac{\beta}{2T}\sum_{t=1}^{T}\sum_{s,a}\lambda_{sa}^t(\Delta_{sa}^t)^2]\overset{(a)}=&\frac{\beta}{2T}\sum_{t=1}^{T}\mbE[\sum_{s,a}\lambda_{sa}^t(\Delta_{sa}^t)^2]\\
	\overset{(b)}=&\frac{\beta}{2T}\sum_{t=1}^{T}\mbE[\mbE[\sum_{s,a}\lambda_{sa}^t(\Delta_{sa}^t)^2|\bbu^t,\bbv^t,\bblambda^t]]
	\\
	\leq&\frac{512\beta|\ccalS||\ccalA|}{(1-\delta)(1-\gamma)^2\varphi^2}\leq \frac{1024\beta|\ccalS||\ccalA|}{(1-\gamma)^2\varphi^2},
	\end{aligned}
\end{equation}
where step (a) holds by the linear of expectation and step (b) holds due to law of total expectation. The last inequality holds by $\delta\in(0,\frac{1}{2})$.
\subsection{Expression for $\mathbb{E}[T_3]$}
For any fixed $\bbu^t, \bbv^t,\bblambda^t$, we have
\begin{equation}
\mbE[\hat{\nabla}_{\bblambda} \ccalL(\bbu^t,\bbv^t,\bblambda^t)|\bbu^t,\bbv^t,\bblambda^t]=\nabla_{\bblambda} \ccalL(\bbu^t,\bbv^t,\bblambda^t)-M\cdot\bbone.
\end{equation}
Thus,
\begin{equation}
\mbE[T_3]=\frac{1}{T}\sum_{t=1}^{T}\mbE[\left<\hat{\nabla}_{\bblambda}\ccalL(\bbw^t,\bblambda^t)-\nabla_{\bblambda}\ccalL(\bbw^t,\bblambda^t),\bblambda^t-\bblambda\right>]=\frac{1}{T}\sum_{t=1}^{T}\mbE[\left<-M\cdot\bbone,\bblambda^t-\bblambda\right>]=0
\end{equation}
where the last step is true because $\left<\bblambda^t\cdot\bbone\right>=\left<\bblambda^*\cdot\bbone\right>=1$
\subsection{Upper bound for $\mathbb{E}[T_4]$}
For any $\bbu\in\ccalU$
\begin{equation}
    \Vert\bbu^1-\bbu\Vert^2\leq \Vert\bbu^1\Vert^2+\Vert\bbu\Vert^2+2|\left<\bbu^1,\bbu\right>|\leq \Vert\bbu^1\Vert^2+\Vert\bbu\Vert^2+2\Vert\bbu^1\Vert\Vert\bbu\Vert\leq \frac{64}{\varphi^2}
\end{equation}
where the last inequality holds by $\Vert\bbx\Vert_2\leq \Vert\bbx\Vert_1$ for any $\bbx$ and the definition of $\ccalU$.
Similarly, for any $\bbv\in\ccalV$
\begin{equation}
    \begin{aligned}
    \Vert\bbv^1-\bbv\Vert^2&\leq \Vert\bbv^1\Vert^2+\Vert\bbv\Vert^2+2\Vert\bbv^1\Vert\Vert\bbv^t\Vert\leq |\ccalS|(\Vert\bbv^1\Vert_\infty^2+\Vert\bbv\Vert_\infty^2+2\Vert\bbv^1\Vert_\infty\Vert\bbv^t\Vert_\infty)\\
    &\leq 16|\ccalS|[\frac{1}{1-\gamma}+\frac{2}{(1-\gamma)\varphi}]^2\leq \frac{144|\ccalS|}{(1-\gamma)^2\varphi^2}
    \end{aligned}
\end{equation}
Finally, combine above two inequalities,
\begin{equation}
    \mbE[T_4]=\frac{1}{2T\alpha}\Vert\bbw^1-\bbw^\dagger\Vert^2=\frac{1}{2T\alpha}[\Vert\bbu^1-\bbu^\dagger\Vert^2+\Vert\bbv^1-\bbv^\dagger\Vert^2]\leq\frac{8}{T\alpha\varphi^2}\bigg[4+\frac{9|\ccalS|}{(1-\gamma)^2}\bigg]\leq\frac{104|\ccalS|}{T\alpha(1-\gamma)^2\varphi^2}
\end{equation}
\subsection{Upper bound for $\mathbb{E}[T_5]$}
For any fixed $\bbu^t,\bbv^t,\bblambda^t$, we have
\begin{equation}\label{eq:bound_T5_1}
    \begin{aligned}
    \mbE\bigg[\Vert\hat{\nabla}_{\bbu}\ccalL(\bbu^t,\bbv^t,\bblambda^t)\Vert^2\bigg|\bbu^t,\bbv^t,\bblambda^t\bigg]&=\mbE_{s_t,a_t}\bigg[\Vert\frac{\lambda^t_{s_ta_t}\bbg_{s_ta_t}}{\zeta_{s_ta_t}}-\kappa\bbone\Vert^2\bigg|\bbu^t,\bbv^t,\bblambda^t\bigg]\\
    &\leq \mbE_{s_t,a_t}\bigg[2\Vert\frac{\lambda^t_{s_ta_t}\bbg_{s_ta_t}}{(1-\delta)\lambda_{s_ta_t}^t+\frac{\delta}{|\ccalS||\ccalA|}}\Vert^2\bigg|\bbu^t,\bbv^t,\bblambda^t\bigg]+2\kappa^2I\\
    &\leq \mbE_{s_t,a_t}\bigg[2\Vert\frac{\bbg_{s_ta_t}}{(1-\delta)}\Vert^2\bigg|\bbu^t,\bbv^t,\bblambda^t\bigg]+2\kappa^2I\\
    &=\frac{2\Vert\bbg_{s_ta_t}\Vert^2}{(1-\delta)^2}+2\kappa^2I\leq \frac{2I}{(1-\delta)^2}+2I 
    \end{aligned}
\end{equation}
where the last step holds because $|g_i(s,a)|\leq 1,\forall i$ by the definition and the fact $0<\kappa\leq 1$. 
\begin{equation}\label{eq:bound_T5_2}
	\begin{aligned}
	\mbE\bigg[\Vert\hat{\nabla}_{\bbv}\ccalL(\bbu^t,\bbv^t,\bblambda^t)\Vert^2\bigg|\bbu^t,\bbv^t,\bblambda^t\bigg]&=\mbE_{s_t,a_t,s_t',s_0}\bigg[\Vert(1-\gamma)\bbe_{s_0}+\frac{\lambda_{s_ta_t}(\gamma\bbe_{s_t'}-\bbe_{s_t})}{\zeta_{s_ta_t}}\Vert^2\bigg|\bbu^t,\bbv^t,\bblambda^t\bigg]\\
	&\leq\mbE_{s_t,a_t,s_t',s_0}\bigg[\Vert(1-\gamma)\bbe_{s_0}+\frac{\lambda_{s_ta_t}(\gamma\bbe_{s_t'}-\bbe_{s_t})}{(1-\delta)\lambda_{s_ta_t}^t+\frac{\delta}{|\ccalS||\ccalA|}}\Vert^2\bigg|\bbu^t,\bbv^t,\bblambda^t\bigg]\\
	&\leq \mbE_{s_t,a_t,s_t',s_0}\bigg[3(1-\gamma)^2+\frac{3\gamma^2+3}{(1-\delta)^2}\bigg]\leq 3+\frac{6}{(1-\delta)^2}
	\end{aligned}
\end{equation}
Combined Eq. \eqref{eq:bound_T5_1}, \eqref{eq:bound_T5_2} with the definition of $\bbw$,
\begin{equation}
    \begin{aligned}
    \mbE[T_5]&=\frac{\alpha}{2T}\sum_{t=1}^{T}\mbE\Vert\hat{\nabla}_{\bbw}\ccalL(\bbw^t,\bblambda^t)\Vert^2=\frac{\alpha}{2T}\sum_{t=1}^{T}\bigg[\mbE\Vert\hat{\nabla}_{\bbu}\ccalL(\bbu^t,\bbv^t,\bblambda^t)\Vert^2+\mbE\Vert\hat{\nabla}_{\bbv}\ccalL(\bbu^t,\bbv^t,\bblambda^t)\Vert^2\bigg]\\
    &\leq \frac{\alpha}{2}\bigg[3+\frac{2I+6}{(1-\delta)^2}+2I\bigg]\leq \frac{37I\alpha}{2}
    \end{aligned}
\end{equation}
where the last step holds by $\delta\in(0,\frac{1}{2})$
\subsection{Upper bound for $\mathbb{E}[T_6]$}\label{first_T_6}
Firstly, notice that $T_6$ is different from $T_3$ because $\bbw^{\dagger}$ depends on $\barblambda$, which is a random variable. However $\bblambda_{\kappa}^*$ depends only on $\kappa$, which is a constant. Thus, in order to bound $T_6$, we need following Lemma.
\begin{lemma}[\cite{beck2017first}]\label{lem:Bach2019}
	Let $\ccalZ\subset \mbR^d$ be a convex set and $\omega:\ccalZ\rightarrow \mbR$ be a $1$-strongly convex function with respect to norm $\Vert\cdot\Vert$ over $\ccalZ$. With the assumption that for all $x\in\ccalZ$ we have $\omega(\bbx)-\min_{\bbx\in\ccalZ}\omega(\bbx)\leq \frac{1}{2}D^2$, then for any martingale difference sequence $\{\bbZ_k\}_{k=1}^{K}\in\mbR^d$ and any random vector $\bbx\in\ccalZ$, it holds that
	\begin{equation}
	\mbE\bigg[\sum_{k=1}^{K}\left<\bbZ_k,\bbx\right>\bigg]\leq \frac{D}{2}\sqrt{\sum_{k=1}^{K}\mbE[\Vert\bbZ_k\Vert_*^2]}
	\end{equation}
	where $\Vert\cdot\Vert_*$ denotes the dual norm of $\Vert\cdot\Vert$
\end{lemma}
For any fixed $\bbu^t,\bbv^t,\bblambda^t$, the gradient estimation is unbiased. 
\begin{equation}
	\mbE[\hat{\nabla}_{\bbphi}\ccalL(\bbu^t,\bbv^t,\bblambda^t)]=\nabla_{\bbphi}\ccalL(\bbu^t,\bbv^t,\bblambda^t)
\end{equation}
where $\bbphi=\bbu\,\text{or}\,\bbv$. Thus,
\begin{align}
	\mbE[T_6]&=\frac{1}{T}\sum_{t=1}^{T}\mbE\bigg[\left<\nabla_\bbw\ccalL(\bbw^t,\bblambda^t)-\hat{\nabla}_\bbw\ccalL(\bbw^t,\bblambda^t),\bbw^t-\bbw^\dagger\right>\bigg]\nonumber\\
	&=\frac{1}{T}\sum_{t=1}^{T}\mbE\bigg[\left<\hat{\nabla}_\bbw\ccalL(\bbw^t,\bblambda^t)-\nabla_\bbw\ccalL(\bbw^t,\bblambda^t),\bbw^\dagger\right>\bigg]\label{eq:boundT6_1}.
\end{align}
To apply Lemma \ref{lem:Bach2019}, let $\ccalZ=\ccalW$, $\omega(\bbx)=\frac{1}{2}\Vert\bbx\Vert^2$, $\bbx=\bbw^\dagger$ and $\bbZ_k=\hat{\nabla}_\bbw\ccalL(\bbw^k,\bblambda^k)-\nabla_\bbw\ccalL(\bbw^k,\bblambda^k)$, which is a martingale difference. Then, $\omega(\bbx)-\min_{\bbx\in\ccalZ}\omega(\bbx)=\omega(\bbw)=\frac{1}{2}\Vert\bbw\Vert^2\leq \frac{1}{2}D^2$ and thus $D\geq \Vert \bbw\Vert$. The norm of $\bbw$ can be bounded as
\begin{equation}
	\begin{aligned}
	\Vert \bbw\Vert^2&=\Vert \bbu\Vert^2+\Vert\bbv\Vert^2\leq \Vert\bbu\Vert_1^2+|\ccalS|\Vert\bbv\Vert_\infty^2=(\frac{4}{\varphi})^2+|\ccalS|\bigg[ \frac{2}{1-\gamma}+\frac{4}{(1-\gamma)\varphi}\bigg]^2\\   
	&\leq \frac{16}{\varphi^2}+\frac{4|\ccalS|}{(1-\gamma)^2}+\frac{16|\ccalS|}{(1-\gamma)^2\varphi^2}+\frac{16|\ccalS|}{(1-\gamma)^2\varphi}\leq \frac{52|\ccalS|}{(1-\gamma)^2\varphi^2}.
	\end{aligned}
\end{equation}
Thus, $\Vert \bbw\Vert\leq \frac{8\sqrt{|\ccalS|}}{(1-\gamma)\varphi}=:D$. Apply Lemma \ref{lem:Bach2019} to Eq. \eqref{eq:boundT6_1},
\begin{equation}
    \begin{aligned}
    \mbE[T_6']&\leq \frac{2\sqrt{
    13|\ccalS|}}{T(1-\gamma)\varphi}\sqrt{\sum_{t=1}^{T}\mbE[\Vert\hat{\nabla}_\bbw\hat{\ccalL}(\bbw^t,\bblambda^t)-\nabla_\bbw\hat{\ccalL}(\bbw^t,\bblambda^t)\Vert^2]}\\
    &\leq \frac{2\sqrt{
    13|\ccalS|}}{T(1-\gamma)\varphi}\sqrt{\sum_{t=1}^{T}\mbE[\Vert\hat{\nabla}_\bbw\hat{\ccalL}(\bbw^t,\bblambda^t)\Vert^2]}\\
    &= \frac{2\sqrt{
    13|\ccalS|}}{T(1-\gamma)\varphi}\sqrt{\frac{2T}{\alpha}\mbE[T_5']}\\
	&\leq \frac{2\sqrt{
    13|\ccalS|}}{\sqrt{T}(1-\gamma)\varphi}\sqrt{37I}\leq\frac{48\sqrt{
    |\ccalS|I}}{\sqrt{T}(1-\gamma)\varphi}.
    \end{aligned}
\end{equation}
\section{Proofs for Section \ref{sec:occupancy_measure}}
\subsection{Proof of Lemma \ref{lem:diff_conservative}} \label{proof_lemdiff_conservative}
\begin{proof}
	Recall $\bblambda^*$ is the optimal occupancy measure to the original problem, which gives
	\begin{equation}
	\left<\bblambda^*,\bbg^i\right>\geq 0
	\end{equation}
	Further, under the Slater Condition Assumption \ref{ass:Slater}, there exists at least one occupancy measure $\tilde{\bblambda}$ such that
	\begin{equation}
	\left<\tilde{\bblambda},\bbg^i\right>\geq \varphi
	\end{equation}
	Define a new occupancy measure $\hat{\bblambda}=(1-\frac{\kappa}{\varphi})\bblambda^*+\frac{\kappa}{\varphi}\tilde{\bblambda}$. It can be shown a feasible occupancy measure to the conservative problem.
	\begin{equation}
	\left<\hat{\bblambda},\bbg^i\right>=\left<(1-\frac{\kappa}{\varphi})\bblambda^*+\frac{\kappa}{\varphi}\tilde{\bblambda},\bbg^i\right>\geq\frac{\kappa}{\varphi}\varphi=\kappa
	\end{equation}
	\begin{equation}
	\sum_{a}(\bbI-\gamma\bbP_a^T)\hat{\bblambda}_a=(1-\frac{\kappa}{\varphi})\sum_{a}(\bbI-\gamma\bbP_a^T)\bblambda^*_a+\frac{\kappa}{\varphi}\sum_{a}(\bbI-\gamma\bbP_a^T)\tilde{\bblambda}_a=(1-\gamma)\tbrho
	\end{equation}
	Then, we can bound the difference
	\begin{equation}
	    \begin{aligned}
	    \left<\bblambda^*,\bbr\right>-\left<\bblambda^*_{\bbkappa},\bbr\right>&\overset{(a)}\leq \left<\bblambda^*,\bbr\right>-\left<\hat{\bblambda},\bbr\right>    =\left<\bblambda^*-(1-\frac{\kappa}{\varphi})\bblambda^*-\frac{\kappa}{\varphi}\tilde{\bblambda},\bbr\right>\\
	    &=\left<\frac{\kappa}{\varphi}\bblambda^*-\frac{\kappa}{\varphi}\tilde{\bblambda},\bbr\right>\overset{(b)}\leq \left<\frac{\kappa}{\varphi}\bblambda^*,\bbr\right>\overset{(c)}\leq \frac{\kappa}{\varphi}\\
	    \end{aligned}
	\end{equation}
	
	{The first step (a) holds because $\bblambda^*_{\bbkappa}$ is the optimal solution of the conservative problem, which gives larger value function than any other feasible occupancy measure. We drop the negative term in the step (b) and the last step (c) is true because $\left<\bblambda^*,\bbr\right>\leq 1$ by the definition.}
\end{proof}

\subsection{Proof of Theorem \ref{thm:occupancy_measure}} \label{proof_thm:occupancy_measure}
\begin{proof}
    In order to construct the relation between duality gap and result in occupancy measure space, let us consider the expression for the Lagrangian function. By the feasibility of
    $\bblambda^*_{\bbkappa}$, we can write
    \begin{equation}\label{eq:saddle_primal}
        \begin{aligned}
        \ccalL(\bbu^t,\bbv^t,\bblambda_{\bbkappa}^*)=&\left<\bblambda^*_{\bbkappa},\bbr\right>+\left<\bbu^t,\bbG^T\bblambda^*_{\bbkappa}-\bbkappa\right>+\bigg[\sum_{a}(\bblambda_{\bbkappa,a}^*)^T(\gamma\bbP_a-\bbI)-(1-\gamma)\bbrho\bigg]\bbv^t
        \\
        \geq& \left<\bblambda^*_{\bbkappa},\bbr\right>.
        \end{aligned}
    \end{equation}
    Define the set $\ccalI=\{i|\left<g_i,\barblambda\right><0\}$. Denote $\bbu'=[u_1',u_2',\cdots,u_{I}']^T$, where $u_i'=u_i$ if $i\in\ccalI$ and $u_i'=0$ otherwise. Define $C_1:=\frac{2}{\varphi}$ and $C_2=\frac{1}{1-\gamma}+\frac{2}{(1-\gamma)\varphi}$ for simplicity, which is the bound for $\Vert\bbu^*_{\kappa}\Vert_{1}$ and $\Vert \bbv^*_{\kappa}\Vert_\infty$, respectively. By the definition of $\bbu^\dagger,\bbv^\dagger$
    \begin{equation}\label{eq:saddle_dual}
        \begin{aligned}
        \ccalL(\bbu^\dagger,\bbv^\dagger,\barblambda)&=\min_{\bbu\in\ccalU,\bbv\in\ccalV}\left<\barblambda,\bbr\right>+\left<\bbu,\bbG^T\barblambda-\bbkappa\right>+\bigg[\sum_{a}(\barblambda_a)^T(\gamma\bbP_a-\bbI)-(1-\gamma)\bbrho\bigg]\bbv\\
        &\overset{(a)}=\min_{\bbu'\in\ccalU,\bbv\in\ccalV}\left<\barblambda,\bbr\right>+\left<\bbu',[\bbG^T\barblambda-\bbkappa]_{-}\right>+\bigg[\sum_{a}(\barblambda_a)^T(\gamma\bbP_a-\bbI)-(1-\gamma)\bbrho\bigg]\bbv,
        \end{aligned}
    \end{equation}
{where the notation $x_{-}:=\min\{x,0\}$ and the equality holds because $u_i=0, i\in\ccalI^c$ for those constraints which are satisfied.} Let us consider the second term on the right hand side of the above expression as follows 
    \begin{align}
        |\left<\bbu',[\bbG^T\barblambda-\kappa\bbone]_{-}\right>|\leq& \Vert\bbu'\Vert_1\Vert[\bbG^T\barblambda-\kappa\bbone]_{-}\Vert_{\infty}\nonumber
        \\
        \leq& 2C_1\Vert[\bbG^T\barblambda-\kappa\bbone]_{-}\Vert_{\infty}.
    \end{align}
      Notice that equality in the above inequality is achievable by selecting $u_j^\dagger=2C_1$ for $j=\argmax_{i}|\left<\barblambda,\bbg^i\right>-\kappa|$ and $u_k^\dagger=0$ for $k\neq j$. Such $\bbu^\dagger$ gives the minimum of $\left<\bbu',[\bbG^T\barblambda-\kappa\bbone]_{-}\right>=2C_1\Vert[\bbG^T\bblambda-\kappa\bbone]_{-}\Vert_{\infty}$. Similarly,  $\bbv^\dagger=2C_2\bbone$ gives the minimum of $\bigg[\sum_{a}(\barblambda_a)^T(\gamma\bbP_a-\bbI)-(1-\gamma)\rho\bigg]\bbv=2C_2\Vert\sum_{a}(\barblambda_a)^T(\gamma\bbP_a-\bbI)-(1-\gamma)\rho\Vert_{1}$ by Holder inequality. Hence, we could write the expression in \eqref{eq:saddle_dual} as
       \begin{equation}\label{eq:saddle_dual2}
        \begin{aligned}
        \ccalL(\bbu^\dagger,\bbv^\dagger,\barblambda)      &=\left<\barblambda,\bbr\right>-2C_1\Vert[\bbG^T\barblambda-\bbkappa]_{-}\Vert_{\infty}-2C_2\Vert\sum_{a}(\barblambda_a)^T(\gamma\bbP_a-\bbI)-(1-\gamma)\bbrho\Vert_{1}.
        \end{aligned}
    \end{equation}
    Combining Eq. \eqref{eq:saddle_dual2} with \eqref{eq:saddle_primal} and then taking expectation, we obtain
    \begin{equation}
        \mbE[\ccalL(\bbu^t,\bbv^t,\bblambda_{\bbkappa}^*)-\ccalL(\bbu^\dagger,\bbv^\dagger,\barblambda)]\geq\mbE\bigg[\left<\bblambda^*_{\bbkappa},\bbr\right> -\left<\barblambda,\bbr\right>+2C_1\Vert[\bbG^T\bblambda^t-\bbkappa]_{-}\Vert_{\infty}+2C_2\Vert\sum_{a}(\bblambda_a^t)^T(\gamma\bbP_a-\bbI)+(1-\gamma)\bbrho\Vert_{1}\bigg].
    \end{equation}
    Combining with the result in Theorem \ref{thm:duality_gap}, there exists a constant $\tdc_1$ such that
    \begin{align}
        \mbE\Bigg[\left<\bblambda^*_{\bbkappa},\bbr\right> -\left<\barblambda,\bbr\right>+&2C_1\Vert[\bbG^T\bblambda^t-\bbkappa]_{-}\Vert_{\infty}+2C_2\Big\|\sum_{a}(\bblambda_a^t)^T(\gamma\bbP_a-\bbI)+(1-\gamma)\bbrho\Big\|_{1}\Bigg]
        \nonumber
        \\
        &\leq \tdc_1\bigg(\sqrt{\frac{I|\ccalS||\ccalA|\log(|\ccalS||\ccalA|)}{T}}\cdot \frac{1}{(1-\gamma)\varphi}\bigg).
    \end{align}
    Denote $L:= \tdc_1\bigg(\sqrt{\frac{I|\ccalS||\ccalA|\log(|\ccalS||\ccalA|)}{T}}\cdot \frac{1}{(1-\gamma)\varphi}\bigg)$. By the Theorem \ref{thm:convergence_bound} (see Appendix \ref{extra_appendix} for reference), we directly get
    \begin{subequations}\label{eq:bound_occupancy_measure_temp}
        \begin{align}
        \mbE[\left<\bblambda_{\kappa}^*,\bbr\right>-\left<\barblambda,\bbr\right>]&\leq L,\label{eq:bounded_conservative_constrain0}
        \\
        \mbE\Vert[\bbG^T\barblambda-\kappa\bbone]_{-}\Vert_{\infty}&\leq \frac{2\epsilon}{C_1}=L\varphi\label{eq:bounded_conservative_constraint1},
        \\
        \mbE\Big\|\sum_{a}(\gamma\bbP_a^T-\bbI)\barblambda_a+(1-\gamma)\bbrho\Big\|_{1}&\leq \frac{2\epsilon}{C_2}\leq(1-\gamma)L\varphi.
        \end{align}
    \end{subequations}
    Note that the result in \eqref{eq:bounded_conservative_constrain0} is at $\bblambda_{\kappa}^*$ and in order to obtain the result for $\bblambda^*$, let us consider  and by the statement of Lemma \ref{lem:diff_conservative}, we could write 
    \begin{equation}\label{eq:bound_occupancy_measure_objective}
        \begin{aligned}
        \mbE[\left<\bblambda^*,\bbr\right>-\left<\barblambda,\bbr\right>]= &\left<\bblambda^*,\bbr\right>- \left<\bblambda_{\kappa}^*,\bbr\right>+\mbE[\left<\bblambda_{\kappa}^*,\bbr\right>-\left<\barblambda,\bbr\right>]
        \\
        \leq& \frac{\kappa}{\varphi}+L,
        \end{aligned}
    \end{equation}
    where we have utilized the upper bound developed in Lemma  \ref{lem:diff_conservative}. Next, recall that $\kappa=2\tdc_1\bigg(\sqrt{\frac{I|\ccalS||\ccalA|\log(|\ccalS||\ccalA|)}{T}}\cdot\frac{1}{1-\gamma}\bigg)$ and from the definition of $L$, we can write
        \begin{equation}\label{eq:bound_occupancy_measure_objective}
        \begin{aligned}
        \mbE[\left<\bblambda^*,\bbr\right>-\left<\barblambda,\bbr\right>]\leq 3\tdc_1\bigg(\sqrt{\frac{I|\ccalS||\ccalA|\log(|\ccalS||\ccalA|)}{T}}\cdot \frac{1}{(1-\gamma)\varphi}\bigg),
        \end{aligned}
    \end{equation}
    which establishes the upper bound for the optimally gap for the original optimization problem.  
    Further, from the result in \eqref{eq:bounded_conservative_constraint1}, we have for all $i\in[I]$
    \begin{equation}
        \mbE|[\left<\barblambda,\bbg^i\right>-\kappa]_{-}|\leq L\varphi.
    \end{equation}
    Note that by the definition of $[x]_{-}:=\min\{x,0\}$, it holds that $|[x]_{-}|=-\min\{x,0\}$ which holds due to the fact that $\min\{x,0\}$ is either zero or negative. Therefore, it holds that $|\left<\barblambda,\bbg^i\right>-\kappa]_{-}|=-[\left<\barblambda,\bbg^i\right>-\kappa]_{-}$ and thus
    \begin{equation}
        \mbE\big([\left<\barblambda,\bbg^i\right>-\kappa]_{-}\big)\geq -L\varphi.
    \end{equation}
    Further, since $[x]_{-}$ is a concave function with respect to $x$, via Jensen's inequality, we can write 
    \begin{equation}\label{here34}
        [\mbE[\left<\barblambda,\bbg^i\right>-\kappa]]_{-}\geq \mbE\big([\left<\barblambda,\bbg^i\right>-\kappa]_{-}\big)\geq -L\varphi.
    \end{equation}
    Again, by the definition of $[x]_{-}$, we simplifies \eqref{here34} to
    \begin{align}
        \min\{\mbE\left<\barblambda,\bbg^i\right>-\kappa,0\}\geq -L\varphi.
    \end{align}
     Thus, we obtain either $\mbE[\left<\barblambda,\bbg^i\right>]\geq\kappa> 0$ or $\mbE[\left<\barblambda,\bbg^i\right>]\geq\kappa -L\varphi$. The first case is trivial and for the second case, recall $\kappa=2\tdc_1\bigg(\sqrt{\frac{I|\ccalS||\ccalA|\log(|\ccalS||\ccalA|)}{T}}\cdot\frac{1}{1-\gamma}\bigg)$
    \begin{equation}\label{eq:bound_occupancy_measure_constraint}
        \mbE\left<\barblambda,\bbg^i\right>\geq\kappa -L\varphi=\tdc_1\bigg(\sqrt{\frac{I|\ccalS||\ccalA|\log(|\ccalS||\ccalA|)}{T}}\cdot\frac{1}{1-\gamma}\bigg)
    \end{equation}
    Let $T=\tdc_1^2\frac{I|\ccalS||\ccalA|\log(|\ccalS||\ccalA|)}{(1-\gamma)^2\varphi^2\epsilon^2}$. By Eq. \eqref{eq:bound_occupancy_measure_temp}, we have the final result
    \begin{subequations}
        \begin{align}
        \mbE[\left<\bblambda^*,\bbr\right>-\left<\barblambda,\bbr\right>]&\leq 3\epsilon\\
        \mbE\left<\barblambda,\bbg^i\right>&\geq \eps\varphi \quad\forall i\in[I]\\
        \mbE\Vert\sum_{a}(\gamma\bbP_a^T-\bbI)\barblambda_a+(1-\gamma)\bbrho\Vert_{1}&\leq(1-\gamma)\epsilon\varphi
        \end{align}
    \end{subequations}
    Recall that it is required $\kappa\leq \min\{\frac{\varphi}{2},1\}$, which gives
    \begin{equation}
        T\geq 4\tdc_1^2\frac{I|\ccalS||\ccalA|\log(|\ccalS||\ccalA|)}{(1-\gamma)^2\varphi^2}\max\{4,\varphi^2\}
    \end{equation}
\end{proof}

\subsection{Proof of Corollary \ref{coro:non_zero_violation_occupancy_measure}} \label{proof_cor_1}
\begin{proof}
    Under the condition that $\kappa=0$, it is obvious that $\bblambda^*=\bblambda_{\kappa}^*$. Thus, we have
    \begin{equation}
        \mbE[\left<\bblambda^*,\bbr\right>-\left<\barblambda,\bbr\right>]\leq L= \tdc_1\bigg(\sqrt{\frac{I|\ccalS||\ccalA|\log(|\ccalS||\ccalA|)}{T}}\cdot \frac{1}{(1-\gamma)\varphi}\bigg)
    \end{equation}
    Furthermore, similar to Eq. \eqref{eq:bound_occupancy_measure_constraint}
    \begin{equation}
        \mbE\left<\barblambda,\bbg^i\right>\geq\kappa -L\varphi=-\tdc_1\bigg(\sqrt{\frac{I|\ccalS||\ccalA|\log(|\ccalS||\ccalA|)}{T}}\cdot \frac{1}{1-\gamma}\bigg)
    \end{equation}
    Let $T=\tdc_1^2\frac{I|\ccalS||\ccalA|\log(|\ccalS||\ccalA|)}{(1-\gamma)^2\varphi^2\epsilon^2}$, we derive the following result
    \begin{subequations}
        \begin{align}
        \mbE[\left<\bblambda^*,\bbr\right>-\left<\barblambda,\bbr\right>]&\leq \eps\\
        \mbE\left<\barblambda,\bbg^i\right>&\geq -\eps\quad\forall i\in[I]\\
        \mbE\Vert\sum_{a}(\gamma\bbP_a^T-\bbI)\barblambda_a+(1-\gamma)\bbrho\Vert_{1}&\leq(1-\gamma)\epsilon\varphi
        \end{align}
    \end{subequations}
\end{proof}

\section{Proofs for Section \ref{sec:policy}}
\subsection{Proof of Theorem \ref{thm:policy}}\label{proof_thm:policy}
\begin{proof}
	By the result in Eq. \eqref{eq:bound_occupancy_measure3} and the definition of $\|\cdot\|_1$, we have
	\begin{equation}\label{eq:dual_result_decompose}
	    \mbE\bigg[\sum_s\Big|\sum_{a'}\bbarlambda_{sa'}-\gamma\sum_{a'}\sum_{s'}P_{a'}(s',s)\barblambda_{s'a'}-(1-\gamma)\rho_s\Big|\bigg]\leq(1-\gamma)\epsilon\varphi.
	\end{equation}
	For each $s\in\ccalS$, let us define 
	\begin{equation}\label{eq:definition_eps_s}
	    \Big|\sum_{a'}\bbarlambda_{sa'}-\gamma\sum_{a'}\sum_{s'}P_{a'}(s',s)\barblambda_{s'a'}-(1-\gamma)\rho_s\Big|=(1-\gamma)\eps_s.
	\end{equation}
{We notice that the left hand side of Eq. \eqref{eq:definition_eps_s} gives the physical meaning of occupancy measure, which can be seen in the following Eq. \eqref{eq:transform_dual1}-\eqref{eq:lambda_physical}}.	Furthermore, Notice that $\eps_s$ is a random variable. It is obvious that $\eps_s\geq 0$ and $\mbE[\sum_s\eps_s]\leq \eps\varphi$ by Eq. \eqref{eq:dual_result_decompose}. Then, define the policy induced by $\bbarlambda$ as $\bbarpi(a|s)=\frac{\bbarlambda_{sa}}{\sum_{a'}\bbarlambda_{sa'}}\geq 0$. Multiply the both sides of Eq. \eqref{eq:definition_eps_s} by $\bbarpi(a|s)$ to obtain
	\begin{equation}\label{eq:transform_dual1}
	    \Big|\bbarlambda_{sa}-\gamma\sum_{a'}\sum_{s'}P_{a'}(s',s)\bbarpi(a|s)\bbarlambda_{s'a'}-(1-\gamma)\rho_s\bbarpi(a|s)\Big|= (1-\gamma)\epsilon_s\bbarpi(a|s),\quad \forall a\in\ccalA, s\in\ccalS.
	\end{equation}
	Now define $\rho_{sa}=\rho_s\bbarpi(a|s)$ which can be considered as the initial distribution for state and action following policy $\bbarpi$. Define $P_{\bbarpi}(s,a,s',a')=P_a(s,s')\cdot\bbarpi(a'|s')$, which can be considered as the transition matrix from current state and action pair $(s,a)$ to next state and action pair $(s',a')$. Furthermore, define $\eps_{sa}=\eps_s\bbarpi(a|s)$ and it is obvious that $\sum_{a}\eps_{sa}=\eps_s$. Then, Eq. \eqref{eq:transform_dual1} can be simplified as
	\begin{equation}
	    \Big|\bbarlambda_{sa}-\gamma\sum_{a'}\sum_{s'}P_{\bbarpi}(s',a',s,a)\bbarlambda_{s'a'}-(1-\gamma)\rho_{sa}\Big|= (1-\gamma)\epsilon_{sa},\quad \forall a\in\ccalA, s\in\ccalS.
	\end{equation}
	With a little abuse of notation $\pm$, we can write
	\begin{equation}\label{eq:dual_scaler_form}
	    \bbarlambda_{sa}-\gamma\sum_{a'}\sum_{s'}P_{\bbarpi}(s',a',s,a)\bbarlambda_{s'a'}= (1-\gamma)(\rho_{sa}\pm\epsilon_{sa}),\quad \forall a\in\ccalA, s\in\ccalS,
	\end{equation}
	where $\pm$ means the left hand side can be equal to $(1-\gamma)(\rho_{sa}+\eps_{sa})$ or $(1-\gamma)(\rho_{sa}-\eps_{sa})$. Next, define $\tbrho\in\mbR^{|\ccalS||\ccalA|}=[\rho_{s_1a_1},\rho_{s_1a_2},\cdots,\rho_{s_{|\ccalS|}a_1},\rho_{s_2a_1},\cdots,\rho_{s_{|\ccalS|}a_{|\ccalA|}}]^T$ as a vector, define $\tbepsilon\in\mbR^{|\ccalS||\ccalA|}=[\eps_{s_1a_1},\eps_{s_1a_2},\cdots,\eps_{s_2a_1},\cdots,\eps_{s_{|\ccalS|}a_{|\ccalA|}}]^T$ as a vector, and define $\bbP_{\bbarpi}\in\mbR^{|\ccalS||\ccalA|\times|\ccalS||\ccalA|}$ as a matrix. Then, we could write the expression in Eq. \eqref{eq:dual_scaler_form} in the following compact form as 
	\begin{equation}
	    \barblambda-\gamma\bbP^T_{\bbarpi}\barblambda= (1-\gamma)(\tbrho\pm\tbepsilon)
	\end{equation}
	Notice that $\Vert\bbP^T_{\bbarpi}\Vert_1=\max_{j}\sum_{i=1}^{|\ccalS||\ccalA|}|\bbP^T_{\bbarpi}(i,j)|=1$ and thus $\Vert\gamma\bbP_{\bbarpi}^T\Vert\leq \gamma$. This means $(I-\gamma\bbP^T_{\bbarpi})$ is invertable and $(I-\gamma\bbP^T_{\bbarpi})^{-1}=\sum_{i=0}^{\infty}\gamma^i(\bbP^T_{\bbarpi})^i$. Thus, we have
	\begin{equation}\label{eq:lambda_physical}
	    \barblambda=(1-\gamma)(\bbI-\gamma\bbP^T_{\bbarpi})^{-1}(\tbrho\pm\tbepsilon).
	\end{equation}
	Rearrange items, take inner-product with $\bbr$ and take absolute value, we have
	\begin{equation}
	    |\left<\barblambda-(1-\gamma)(\bbI-\gamma\bbP^T_{\bbarpi})^{-1}\tbrho,\bbr\right>|= (1-\gamma)\left<(\bbI-\gamma\bbP^T_{\bbarpi})^{-1}\tbepsilon,\bbr\right>.
	\end{equation}
	Notice that
	\begin{equation}
	    \left<(1-\gamma)(\bbI-\bbP^T_{\bbarpi})^{-1}\tbrho,\bbr\right>=\tbrho^T\bbr+\gamma\tbrho^T\bbP_{\bbarpi}\bbr+\gamma^2\tbrho^T(\bbP_{\bbarpi})^2\bbr+\cdots=\mbE_{s_0\sim\tbrho}[V_r^{\bbarpi}(s_0)]=J_{r,\rho}(\bbarpi)
	\end{equation}
	The above equation can be bounded by
	\begin{equation}
	    \begin{aligned}
	    \mbE|\left<\barblambda,\bbr\right>-\mbE_{s_0\sim\tbrho}[V_r^{\bbarpi}(s_0)]|&\leq \mbE(1-\gamma)|\left<(\bbI-\gamma\bbP^T_{\bbarpi})^{-1}\tbepsilon,\bbr\right>|\\
	    &\overset{(a)}\leq (1-\gamma)\mbE\Vert(\bbI-\gamma\bbP^T_{\bbarpi})^{-1}\tbepsilon\Vert_{1}\Vert\bbr\Vert_{\infty}\\
	    &\overset{(b)}\leq (1-\gamma)\Vert(\bbI-\gamma\bbP^T_{\bbarpi})^{-1}\Vert_1\mbE\Vert\tbepsilon\Vert_{1}\\
	    &\overset{(c)}\leq (1-\gamma)\sum_{i=0}^{\infty}\Vert\gamma^i(\bbP^T_{\bbarpi})^i\Vert_1\eps\varphi\\
	    &\overset{(d)}\leq (1-\gamma)\sum_{i=0}^{\infty}\gamma^i\eps\varphi=\eps\varphi,
	    \end{aligned}
	\end{equation}
	where step (a) holds by Holder inequality, step (b) holds by definition of matrix norm and reward, step (c) holds by triangle inequality and $\mbE\Vert \tbepsilon \Vert_1=\mbE[\sum_s\eps_s]\leq \eps\varphi$. The last step (d) is true because $\Vert\bbP^T_{\bbarpi}\Vert_1=1$. Finally, we get the result
	\begin{equation}\label{eq:objective_error}
	    \mbE[\left<\barblambda,\bbr\right>-J_{r,\rho}(\bbarpi)]\leq \eps\varphi.
	\end{equation}
	Recall $\left<\bblambda^*,\bbr\right>-\left<\barblambda,\bbr\right>\leq 3\eps$ in Eq. \eqref{eq:bound_occupancy_measure4}, hence we can write 
	\begin{align}
	    J_{r,\rho}(\pi^*)-\mbE[J_{r,\rho}(\bbarpi)]=&\big(\left<\bblambda^*,\bbr\right>-\mbE\left<\barblambda,\bbr\right>\big)+\mbE\big[\left<\barblambda,\bbr\right>-J_{r,\rho}(\bbarpi)\big]
	    \nonumber
	    \\
	    \leq& 4\eps,
	\end{align}
	which is for the objective suboptimality gap in the primal domain. Rescaling $\eps$ to $\frac{\eps}{4}$ finishes the proof. Similarly, for the constraints in the primal domain, we could write
	\begin{equation}\label{eq:constraint_error}
	    \mbE[J_{g_i,\rho}(\bbarpi)-\left<\barblambda,\bbg\right>]\geq -\eps\varphi.
	\end{equation}
From the result in Eq. \eqref{eq:bound_occupancy_measure2}, note that we have $\mbE\left<\barblambda,\bbg\right>\geq \epsilon\varphi$. Hence, after rearranging the terms in \eqref{eq:constraint_error}, we obtain 
	\begin{align}
	    \mbE[J_{g_i,\rho}(\bbarpi)]\geq&-\eps\varphi+\mbE\big[\left<\barblambda,\bbg\right>\big]\nonumber
	    \\
	    =&-\eps\varphi+\eps\varphi\nonumber
	    \\
	    =&0.
	\end{align}
	Hence proved. 
\end{proof}

\subsection{Proof of Corollary \ref{coro:non_zero_violation_policy}} \label{proof_cor_2}
\begin{proof}
    Recall the result in Eq. \eqref{eq:bound_occupancy_measure1} and \eqref{eq:objective_error}, we directly have
    \begin{equation}
        J_{r,\rho}(\pi^*)-\mbE[J_{r,\rho}(\bbarpi)]\leq 2\eps
    \end{equation}
    Similarly, combine Eq. \eqref{eq:bound_occupancy_measure2} and \eqref{eq:constraint_error}, we have
    \begin{equation}
        \mbE[J_{g_i,\rho}(\bbarpi)]\geq -2\eps
    \end{equation}
    Re-scaling $\eps$ to $\frac{\eps}{2}$ finishes the proof.
\end{proof}
\section{Optimization Theory} \label{extra_appendix}
Consider the standard optimization problem
\begin{equation}\label{eq:opt_primalproblem}
    f_{opt}=\min_{x\in\ccalX}\{f(\bbx):\bbg(\bbx)\leq 0,\bbA\bbx+\bbb=0\}
\end{equation}
where $\bbA\in\mbR^{d*n}$, $\bbb\in\mbR^{d}$, $\bbx\in\mbR^n$ and $\bbg:\mbR^{n}\rightarrow\mbR^{m}$. Define the value function as
\begin{equation}
    p(\bbu,\bbt)=\min_{\bbx\in\ccalX}\{f(\bbx):\bbg(\bbx)\leq\bbu,\bbA\bbx+\bbb=\bbt\}
\end{equation}
and the dual function as
\begin{equation}
    q(\bby,\bbz)=\min_{\bbx\in\ccalX}\{f(\bbx)+\bby^T\bbg(\bbx)+\bbz^T(\bbA\bbx+\bbb)\}, \bby\in\mbR^{m}_{+}, \bbz\in\mbR^{d}
\end{equation}
Then the dual problem can be written as 
\begin{equation}\label{eq:opt_dualproblem}
    q_{opt}=\max_{\bby\in\mbR^{m}_{+},\bbz\in\mbR^{d}} q(\bby,\bbz)
\end{equation}
\begin{lemma}(Theorem 3.59 in \cite{beck2017first})\label{lem:opt_condition}
    (\bby,\bbz) is an optimal solution of problem Eq. \eqref{eq:opt_dualproblem} if and only if $-(\bby,\bbz)\in\partial p(\bbzero,\bbzero)$
\end{lemma}
\begin{theorem}(Theorem 3.60 in \cite{beck2017first})\label{thm:convergence_bound}
    Let $f,\bbg$ be convex functions, $\ccalX$ a nonempty convex set, $\bbA\in\mbR^{d*n}$ and $\bbb\in\mbR^{d}$. Let $f_{opt}$, $q_{opt}$ be the optimal values of the primal and dual problems Eq. \eqref{eq:opt_primalproblem} and \eqref{eq:opt_dualproblem}, respectively. Suppose that $f_{opt}=q_{opt}$ and that the optimal set of the dual problem is nonempty. Let $(\bby^*,\bbz^*)$ be the optimal solution of the dual problem, Assume that $\tbx\in\ccalX$ satisfies 
    \begin{equation}
        f(\tbx)-f_{opt}+C_1\Vert\bbg(\tbx)_{+}\Vert_{\infty}+C_2\Vert\bbA\tbx+\bbb\Vert_{1}\leq \delta
    \end{equation}
    where $\delta>0$ and $C_1,C_2$ are constants satisfying $C_1\geq 2\Vert\bby^*\Vert_1$, $C_2\geq 2\Vert \bbz^*\Vert_{\infty}$, then
    \begin{equation}
        \begin{aligned}
            f(\tbx)-f_{opt}&\leq\delta\\
            \Vert \bbg(\tbx)_{+}\Vert_\infty&\leq\frac{2\delta}{C_1}\\
            \Vert\bbA\tbx+\bbb\Vert_{1}&\leq\frac{2\delta}{C_2}
        \end{aligned}
    \end{equation}
\end{theorem}
\begin{proof}
    It is trivial that $f(\tbx)-f_{opt}\leq \delta$ due to the fact that $C_1\Vert\bbg(\tbx)_{+}\Vert_{\infty}$ and $C_2\Vert\bbA\tbx+\bbb\Vert_{1}$ are both non-negative.
    Since $(\bby^*,\bbz^*)$ is the optimal solution for the dual problem, it follows by Lemma \ref{lem:opt_condition} that $-(\bby^*,\bbz^*)\in(\bbzero,\bbzero)$. Therefore, for any $(\bbu,\bbt)\in dom(p)$
    \begin{equation}\label{eq:opt_bound1}
        p(\bbu,\bbt)-p(\bbzero,\bbzero)\geq \left<-\bby^*,\bbu\right>+\left<-\bbz^*,\bbt\right>
    \end{equation}
    Plugging $\bbu=\tbu:=[\bbg(\tbx)]_{+}$ and $\bbt=\tbt:=\bbA\tbx+\bbb$ into Eq. \eqref{eq:opt_bound1}, while using the inequality $p(\tbu,\tbt)\leq f(\tbx)$ and the equality $p(\bbzero,\bbzero)=f_{opt}$, we obatin
    \begin{equation}
        \begin{aligned}
            (C_1-\Vert\bby^*\Vert_{1})\Vert\tbu\Vert_\infty+(C_2-\Vert\bbz^*\Vert_\infty)\Vert\tbt\Vert_1&=-\Vert\bby^*\Vert_{1}\Vert\tbu\Vert_{\infty}-\Vert\bbz^*\Vert_{\infty}\Vert\tbt\Vert_{1}+C_1\Vert\tbu\Vert_{\infty}+C_2\Vert\tbt\Vert_{1}\\
            &\leq \left<-\bby^*,\tbu\right>+\left<-\bbz^*,\tbt\right>+C_1\Vert\tbu\Vert_{\infty}+C_2\Vert\tbt\Vert_{1}\\
            &\leq p(\tbu,\tbt)-p(\bbzero,\bbzero)+C_1\Vert\tbu\Vert_{\infty}+C_2\Vert\tbt\Vert_{1}\\
            &\leq f(\tbx)-f_{opt}+C_1\Vert\tbu\Vert_{\infty}+C_2\Vert\tbt\Vert_{1}\\
            &\leq \delta
        \end{aligned}
    \end{equation}
    It is clear that $C_1-\Vert\bby^*\Vert_{1}$ and $C_2-\Vert\bbz^*\Vert_\infty$ are both non-negative. Thus,
    \begin{equation}
        \begin{aligned}
            (C_1-\Vert\bby^*\Vert_{1})\Vert\tbu\Vert_\infty&\leq \delta\\
            (C_2-\Vert\bbz^*\Vert_\infty)\Vert\tbt\Vert_1&\leq \delta
        \end{aligned}
    \end{equation}
    Finally, using the assumption $C_1\geq 2\Vert\bby^*\Vert_1$, $C_2\geq 2\Vert \bbz^*\Vert_{\infty}$
    \begin{equation}
        \begin{aligned}
            \Vert[g(\tbx)_{+}]\Vert_\infty=\Vert\tbu\Vert_\infty\leq \frac{\delta}{C_1-\Vert\bby\Vert_1}&\leq \frac{2\delta}{C_1}\\
            \Vert[\bbA\tbx+\bbb]\Vert_1=\Vert\tbt\Vert_1\leq \frac{\delta}{C_2-\Vert\bbz\Vert_\infty}&\leq \frac{2\delta}{C_2}\\
        \end{aligned}
    \end{equation}
\end{proof}

\end{document}